\documentclass{article} 
\usepackage{iclr2027_conference,times}

\usepackage{amsmath,amsfonts,bm}

\def\eqref#1{equation~\ref{#1}}

\def\1{\bm{1}}

\DeclareMathAlphabet{\mathsfit}{\encodingdefault}{\sfdefault}{m}{sl}
\SetMathAlphabet{\mathsfit}{bold}{\encodingdefault}{\sfdefault}{bx}{n}

\usepackage{hyperref}
\usepackage{url}

\usepackage{rotating}
\usepackage{booktabs}
\usepackage{amsmath,amssymb} 
\usepackage{color}
\usepackage{algorithm}
\usepackage{algorithmic}
\usepackage{graphicx}
\usepackage{multirow}
\usepackage{subfigure}
\usepackage{comment}
\usepackage{stackengine}
\usepackage{colortbl}
\usepackage{amsthm}     
\usepackage{caption}
\usepackage{wrapfig}
\usepackage{physics}
\usepackage{amssymb}
\usepackage{amsmath}
\usepackage{mathtools}
\usepackage{cleveref}
\usepackage[most]{tcolorbox}
\usepackage{enumitem}

\newtheorem{definition}{Definition}

\newtheorem*{proposition*}{Proposition}
\newtheorem{observation}{Observation}
\crefname{observation}{Observation}{Observations}
\Crefname{observation}{Observation}{Observations}

\newtheorem*{theorem*}{Theorem}



\def\d{\text{d}}

\newcommand{\bw}{\boldsymbol{w}}

\newcommand{\bB}{\boldsymbol{B}}

\newcommand{\bJ}{\boldsymbol{J}}

\newcommand{\bx}{\boldsymbol{x}}

\newcommand{\btheta}{\boldsymbol{\theta}}

\theoremstyle{remark}

 \title{Advancement or Acceleration?\\
How Data Curation Improves LLM Training}

\author{Jiaxing Wang\thanks{Equal Contribution.}\hspace{1ex}$^{1}$, \quad Zicheng Zhang\footnote[1]{}\hspace{1ex}$^{1}$, \quad Deping Xiang$^{1}$, \quad Jin Xu$^{2}$\\   
\textbf{Yong Li$^{1}$, \quad Chao Liu$^{1}$, \quad Pengzhang Liu$^{1}$, \quad Qixia Jiang$^{1}$} \\
$^1$JD.com \quad \quad $^2$University of Oxford \\
\texttt{\{wangjiaxing41, zhangzicheng6, xiangdeping1,   liyong5\}@jd.com} \\
\texttt{\{liuchao397, liupengzhang, jiangqixia\}@jd.com} \\
\texttt{jin.xu@stats.ox.ac.uk} 
}

\iclrfinalcopy 
\begin{document}

\maketitle

\begin{abstract}
While data curation is widely credited with improving large language model training, a basic question remains unresolved: does curation expand what a model can ultimately learn, or reach the same solution faster? In this paper, we formalize this distinction between \emph{advancement} (improving the attainable performance) and \emph{acceleration} (reaching the same performance with less data). Through an analytic characterization and real-world LLM experiments, we show that advancement requires a correctable train–target gap. When the baseline training and target objectives coincide, no curation policy improves the endpoint, and its benefit is acceleration. This reframes advancement as a property of the curation task rather than the curation method. Under a common gap-centric view of data mixture and selection, we operationalize this condition as the $\Delta$-diagnostic, an observable discrepancy between the uncurated baseline's training and target losses that is computable prior to expensive curation. Across extensive experiments, an appreciably nonzero $\Delta$ predicts persistent gains, whereas gains near $\Delta=0$ tend to diminish as the training budget grows. Beyond individual methods, our findings recenter data curation on the task itself and point to target-set construction as a promising direction for shaping what curation can achieve.


\end{abstract}

\section{Introduction}
\label{sec:intro}

In recent years, as the cost-effectiveness of pure scaling declines, how well we use the data is drawing increasing attention. As a primary approach to this, data curation takes two main forms: \emph{data mixture}, which balances the proportions of domain-specific data in training sets to ensure intended capabilities~\citep{xie2023doremi,fan2024doge,chen2025aioli,liu2025regmix}, and \emph{data selection}, which chooses informative training examples that best improve performance~\citep{li2024dclm,liu2024deita,lin2024rho1,yu2024mates}.  Remarkable gains have been reported, and pruning data has even been argued to \emph{beat the power-law scaling} \citep{sorscher2022beyondsl}. These results suggest that data curation is a promising direction for improving models, or so it would seem.

Nevertheless, when compared carefully, the advantage of data curation sometimes becomes surprisingly fragile. \citet{chen2025aioli} shows that optimized mixtures do not always outperform the uniform allocation, while elaborate selection rules at scale may barely exceed, and at times underperform, random sampling~\citep{xia2024random}. 

These observations appear contradictory because performance at a fixed training budget conflates two effects. Curation may improve the best solution the model can attain, or it may reduce the data needed to reach the same solution. We call the former \emph{advancement} and the latter \emph{acceleration}, leading to our central question:
\emph{Does data curation improve the attainable optimum, or only reach the same performance faster?}

We take a gap-centric account of this question, and argue that advancement depends on a property of the task: the \emph{train--target gap} between the distribution induced by the uncurated baseline and the target we evaluate on. Data curation can advance a model only when such a gap exists and can be corrected. Without one, no allocation improves on the attainable performance, though curation may still accelerate learning. To understand this mechanism, we derive an analytic characterization in a simplified but instructive setting: perceptron learning in the teacher--student framework, where the performance of the converged model on the target under any allocation is determined by a closed system of self-consistent equations.

Our analysis gives a common account of data mixture and selection: curation can advance a model only when it corrects a train--target mismatch. Without such a mismatch, techniques that prioritize informative examples or rebalance domains can accelerate learning, but their advantages diminish as the sample budget grows. Because whether advancement is possible depends on the \emph{task} rather than on the curation algorithm, it can be assessed in advance. We introduce the $\Delta$-diagnostic, a discrepancy between target and baseline-training losses, as a lightweight diagnostic of advancement potential. It is computed from a single baseline run, before any compute is committed to data curation.

We validate this theoretical account on real LLM training across data mixture and selection settings. The empirical results confirm our predictions: persistent gains at the largest tested budgets emerge only when a correctable gap is present. Conversely, in regimes where $\Delta \approx 0$, the benefits of curation tend to diminish with additional training, reflecting mere acceleration. This offers an explanation for recent findings that appear to conflict, such as the supervising competitiveness of uniform mixtures~\citep{chen2025aioli} and random selection~\citep{xia2024random}: their empirical settings are close to near-zero-gap regimes. Our findings therefore suggest that advancement potential should be diagnosed before committing to increasingly complex allocation methods. More broadly, they point to target design as an important but underexplored complement to curation for future research since it shapes the train--target gap that curation can correct, dictating the improvement it can ultimately aim for.

Our contributions are threefold:
\begin{itemize}[leftmargin=*]
    \item We distinguish two distinct effects of data curation---\emph{advancement}, which improves the attainable performance, and \emph{acceleration}, which reaches the same performance faster. We show that advancement requires a correctable train--target gap.

   \item We propose the $\Delta$-diagnostic, a lightweight diagnostic of advancement potential. This indicator can be computed from a single baseline run, before committing to expensive curation.

    \item We validate our account on LLM training across data mixture and selection, confirming the theoretical predictions and unifying previously conflicting empirical findings under a single gap-dependent picture.
\end{itemize}

\section{Preliminary}
\label{sec:preliminary}

\paragraph{Curation as Data Allocation} 
Data curation improves models by reallocating the training budget across candidate data. We unify data mixture and selection as a single allocation operation $\bw\in\mathcal W$. 
Consider a training corpus organized into $K$ domains, where domain $k$ provides a
candidate training set $\mathcal D_{\mathrm{train}}^k$ and a target set $\mathcal D_{\mathrm{tar}}^k$. Training under an allocation $\bw$ yields a model $\btheta^\star(\bw)=\arg\min_{\btheta}\mathcal L_{\mathrm{train}}(\btheta;\bw)$. The two paradigms are of different granularities:
\begin{itemize}[leftmargin=*,itemsep=1pt]
\item \emph{Data mixture} ($\bw\in \mathcal{W}:= \{\bw \in \mathbb{R}^K \mid \sum_{k=1}^K w_k=1, w_k > 0 \}$): domain $k$ contributes a fraction $w_k$ of
the sampled training data.
\item \emph{Data selection} ($K=1$, $\bw \in \mathcal{W}:= \left\{\bw\in \{0,1\}^P \mid \sum_{p=1}^{P} w_p= f P \right\}$): A per-example keep decision retaining a subset of size $fP$ from $P$ total examples.
\end{itemize}
We assess $\btheta^\star(\bw)$ by its performance on the target, 
which is the distribution we care about at test time. For mixture the target is typically a domain composition $\widehat{\bw}$ over the target sets $\{\mathcal{D}_{\mathrm{tar}}^k\}$. Unlike the allocation $\bw$, $\widehat{\bw}$ is fixed by the task itself, reflecting how much each domain matters at test time, rather than chosen by curation. For selection, the target is simply the held-out distribution, with no domain-level composition to specify.

\paragraph{What Should Curation Beat?} To evaluate a curation choice, we must compare it against an \emph{uncurated baseline} $\bw_{\text{baseline}}$ that performs no re-allocation. For mixture this is typically uniform allocation $\bw_{\text{baseline}}=\frac{1}{K}\mathbf{1}$. For selection, it is the full retention $f=1$ and $\bw_{\text{baseline}} = \mathbf{1}$. Despite a fast-growing arsenal of curation methods (DoGE~\citep{fan2024doge},  TANDEM~\citep{wang2025tandem}, RHO-1~\citep{lin2024rho1}, BLADE~\citep{wang2026blade}, among others), these simple baselines remain stubbornly hard to beat in some cases.

\section{An Analytic Theory of Data Curation}
\label{sec:analytic_theory}
To understand why these baselines are so hard to beat, we recognize that a curation choice can have two different effects: \emph{advancement}, yielding a genuinely better attainable model, or \emph{acceleration}, reaching the same performance faster. Distinguishing them requires knowing how the learned model behaves under a certain allocation, and how this behavior changes as the sample budget grows. We make this precise with an analytic characterization based on statistical-mechanics, and show that advancement requires a correctable train--target gap. Otherwise, curation can only accelerate.

\subsection{The Teacher--Student Curation Model}
\label{sec:replica_model}

We model data curation using the perceptron in a teacher--student setting~\citep{engel2001statisticalml}. In the high-dimensional limit $N, P \to \infty$ with a data budget $\alpha = P/N$, the student $\bJ^\star(\bw) = \arg\min_{\bJ} \mathcal L_{\mathrm{train}}(\bJ;\bw)$ attains a target generalization performance $\mathcal E_\alpha(\bw)$~\footnote{Following the statistical-mechanics convention $\bJ$ denotes the student parameters and corresponds to the general model parameter vector $\btheta$ used above.}. Tracking $\mathcal E_\alpha(\bw)$ across $\alpha$ allows us to separate advancement from acceleration. For any $\bw$, $\mathcal E_\alpha(\bw)$ is characterized analytically via the replica method~\citep{bruno2021learningcurve,sorscher2022beyondsl}, yielding self-consistent equations that we solve numerically (derivations deferred to Appendix~\ref{app:mixture} and~\ref{app:imperfect}).

Specifically, Domain $k$ provides training data $\mathcal D_{\mathrm{train}}^k=\{(\bx_i^k,y_i^k)\}_{i=1}^{P_k}$, where $\bx_i^k\sim\mathcal N(\mathbf 0,\mathbf\Sigma_k)$ and labels $y_i^k=\operatorname{sign}(\bB_k^\top\bx_i^k)$ are generated by a domain teacher $\bB_k\in\mathbb R^N$. 
We instantiate the two curation paradigms as follows:

\begin{itemize}[leftmargin=*,itemsep=1pt]
\item \emph{Data mixture:} Without loss of generality, we consider two domains $K=2$ and let the policy $\bw=(w_1,w_2)$ allocate the budget $P$ by sampling $P_k=w_kP$ examples from domain $k$. The target performance is the weighted loss $\mathcal E(\bJ)=\sum_{k=1}^2\widehat w_k\,\mathcal L_{\mathrm{tar}}^k(\bJ)$ under a fixed domain composition $\widehat{\bw}$. The train--target gap is directly governed by the mismatch between $\bw_{\mathrm{baseline}}$ and $\widehat{\bw}$.

\item \emph{Data selection:}  This setting considers $K=1$, we evaluate generalization along a distinct \emph{target teacher} $\bB_{\text{tar}}$ at an angle $\phi$ from the training teacher $\bB$. This angle $\phi$ defines the train--target gap. 
From a candidate pool, a target-aware selection rule retains a fraction $f$ of the samples based on a margin criterion along $\bB_{\text{tar}}$, while preferentially filtering out examples whose training labels conflict with $\bB_{\text{tar}}$. This rule serves as a controlled gap-correcting policy: testing it at $\phi=0$ and $\phi>0$ isolates the benefit of gap correction from mere target awareness. The selected data yields a student with target error $\mathcal{E}(\bJ) = \arccos(\rho)/\pi$, where $\rho = \frac{\bJ^{\top} \bB_{\text{tar}}}{\| \bJ \| \| \bB_{\text{tar}} \|}$ is the student--target overlap.
\end{itemize}

\begin{figure}[t]
\centering
\begin{minipage}{0.325\linewidth}
    \centering
    \includegraphics[width=1.02\textwidth]{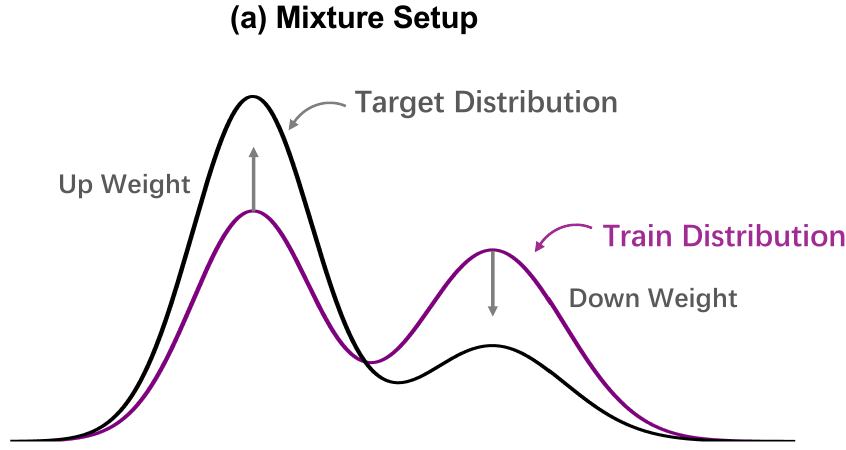}
\end{minipage}
\begin{minipage}{0.325\linewidth}
    \centering
    \includegraphics[width=1.0\linewidth]{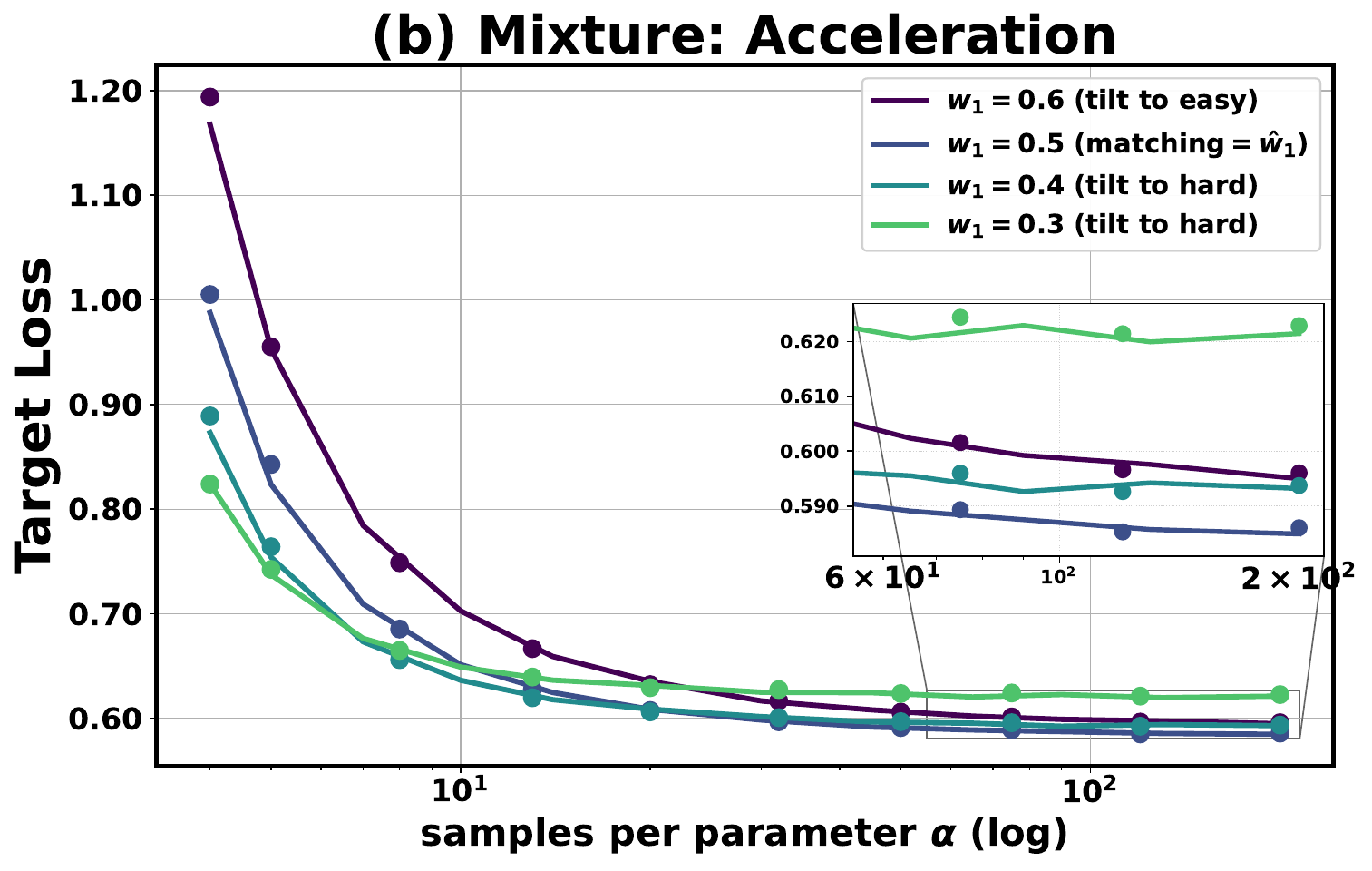}
\end{minipage}
\begin{minipage}{0.325\linewidth}
    \centering
    \includegraphics[width=1.0\linewidth]{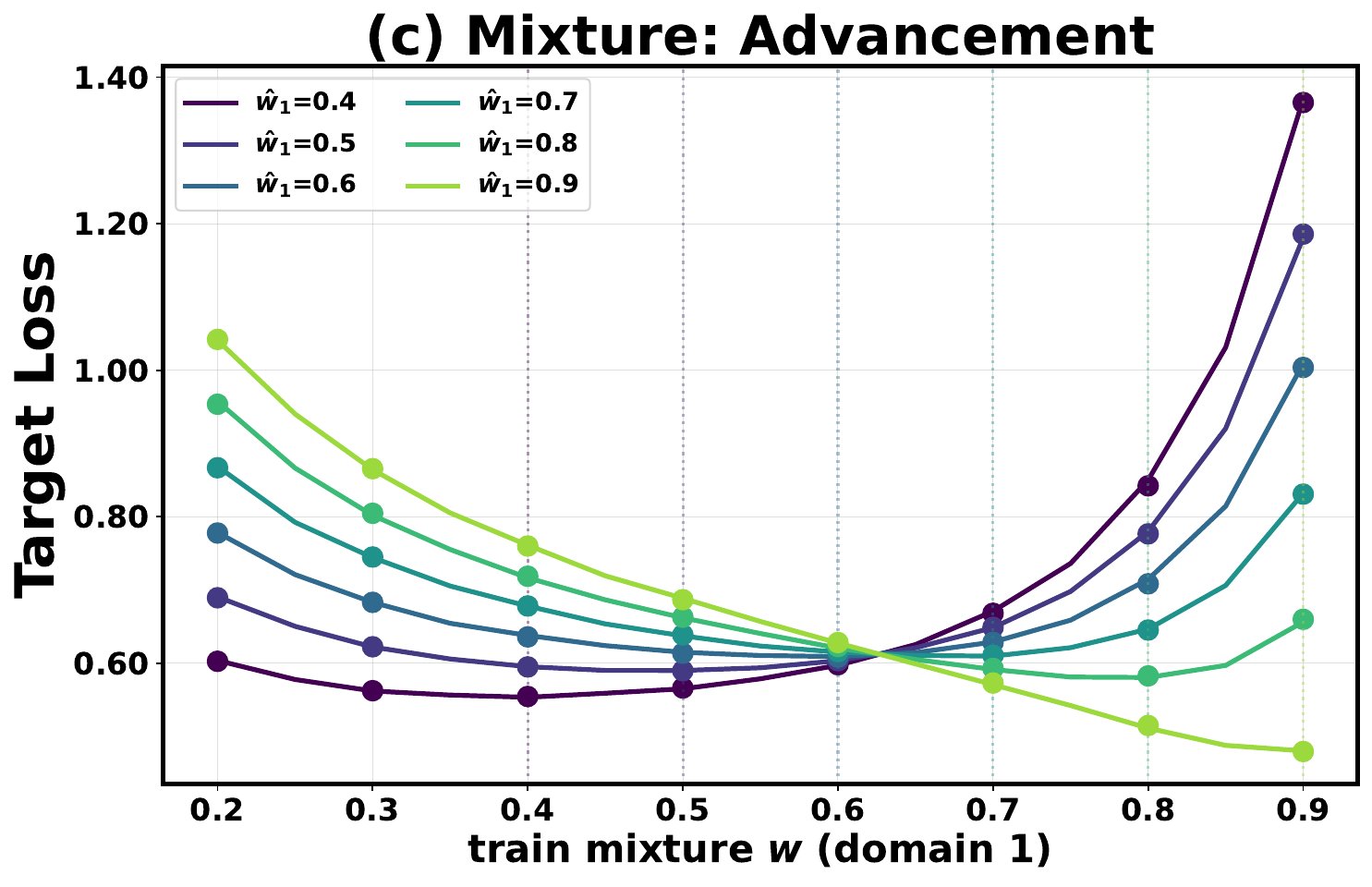}
\end{minipage}
\begin{minipage}{0.325\linewidth}
    \centering
    \includegraphics[width=1.0 \textwidth]{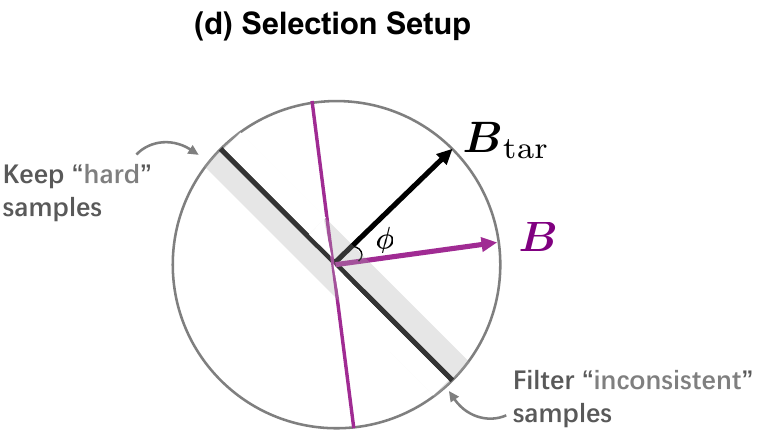}
\end{minipage}
\begin{minipage}{0.325\linewidth}
    \centering
    \includegraphics[width=0.98\linewidth]{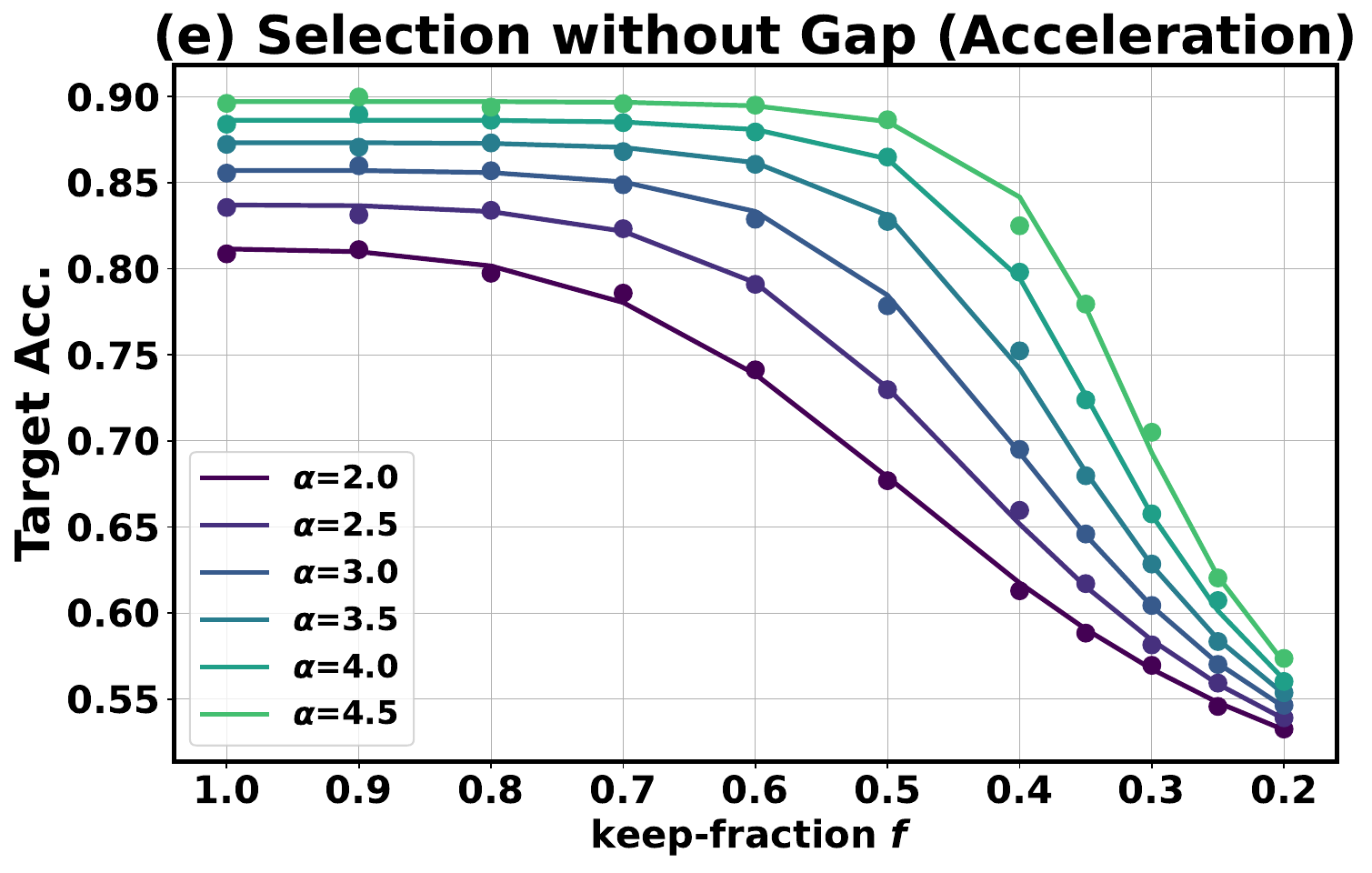}
\end{minipage}
\begin{minipage}{0.325\linewidth}
    \centering
    \includegraphics[width=0.96\linewidth]{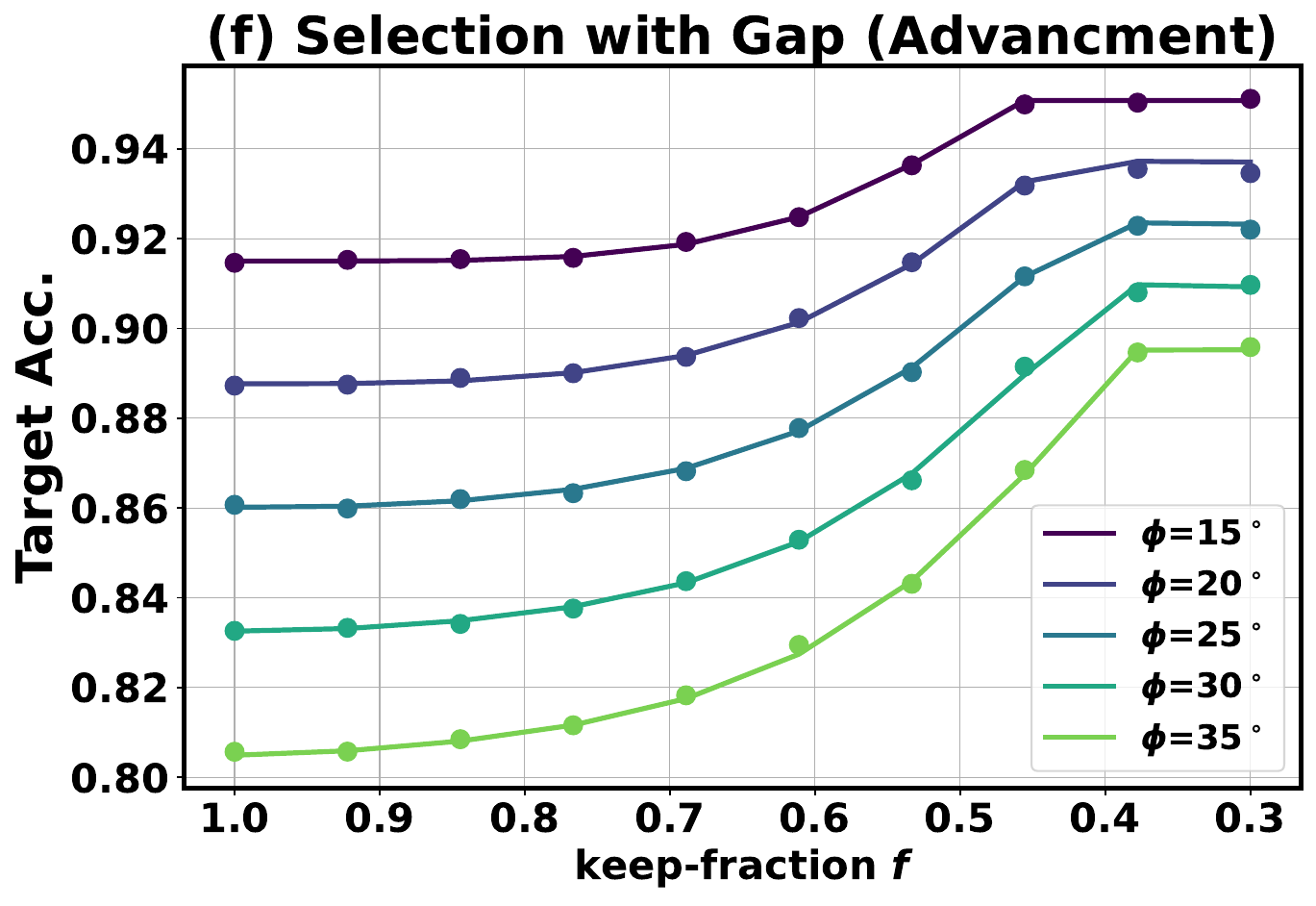}
\end{minipage}
\caption{
An analytic account of data curation. Theory (solid lines) closely matches $N=200$ perceptron simulations (circles). \emph{Top row---mixture.} (a) \emph{The setup.} Two-domain training and target mixtures with compositions $\bw$ and $\widehat{\bw}$. (b) \emph{Acceleration.} Tilting the training mixture toward the harder domain ($w_1<\widehat w_1=0.5$) lowers target loss at small budgets, but the gain vanishes as $\alpha$ grows. (c) \emph{Advancement.} At a large budget~($\alpha=100$), each target loss vs. mixture curve is minimized by target matching, $\bw^\star=\widehat{\bw}$. Appendix~\ref{app:mixture} shows that this holds as $\alpha\to\infty$. \emph{Bottom row---selection.} (d) \emph{The setup.} A selection policy favors examples near the target boundary and filters those whose labels conflict with the target. (e) \emph{Acceleration.} At zero gap ($\phi=0$), moderate selection preserves accuracy, but never exceeds full retention ($f=1$). Aggressive selection degrades performance. (f) \emph{Advancement.} With $\phi>0$, target accuracy initially rises as $f$ decreases: improved train--target alignment outweighs the reduction in sample size. Reaches an interior optimum $f^\star<1$, and the gain increases with $\phi$. Appendix~\ref{app:imperfect} shows that an interior selection policy remains better than full retention as $\alpha \to \infty$.}
\label{fig:mixture_ratio_evolve}
\end{figure}

\subsection{Advancement and Acceleration}
\label{sec:accopt}

$\mathcal E_\alpha(\bw)$ traces target performance across data budgets $\alpha$, revealing two distinct effects of curation: changing the \emph{endpoint}, i.e., the best attainable performance, or accelerating the \emph{approach}, i.e., reducing the budget required to reach a given performance level. We refer to these effects as \emph{advancement} and \emph{acceleration}, respectively.

\begin{definition}[Advancement and acceleration]
\label{def:adv_acc}
Let $\mathcal E_\infty(\bw)=\lim_{\alpha\to\infty}\mathcal E_\alpha(\bw)$ denote the limiting target performance. An allocation provides \emph{advancement} if it improves the endpoint, $\mathcal E_\infty(\bw)<\mathcal E_\infty(\bw_{\text{baseline}})$. For a performance level $\varepsilon$ attainable by both allocations, let $\alpha_{\bw}(\varepsilon)=\inf\{\alpha:\mathcal E_\alpha(\bw)\leq\varepsilon\}$ be the budget at which $\bw$ first reaches $\varepsilon$. The allocation provides \emph{acceleration} at $\varepsilon$ if it reaches this level with fewer examples, $\alpha_{\bw}(\varepsilon)<\alpha_{\bw_{\text{baseline}}}(\varepsilon)$.
\end{definition}

These effects are not mutually exclusive, but they differ in persistence: advancement persists as the budget grows, whereas acceleration yields transient gains that eventually vanish. The analytical characterization of these population endpoints are given in Appendix~\ref{app:mixture} and Appendix~\ref{app:imperfect} for mixture and selection. respectively.

\paragraph{When is advancement possible?}
The distinction above raises a structural question: when can curation improve the attainable target performance, rather than reduce the data budget needed to reach a given performance level? A necessary condition is a \emph{correctable train--target gap}---a mismatch between the training distribution induced by the baseline, $\mathcal P_{\mathrm{train}}(\bw_{\mathrm{baseline}})$, and the target distribution $\mathcal P_{\mathrm{target}}$ that some feasible policy can reduce. Without such a mismatch, curation has nothing to correct, and its benefit is limited to acceleration. A correctable gap instead provides optimization headroom for advancement, although it does not guarantee improvement.
\Cref{sec:results} formalizes this condition and instantiates it for data mixture and data selection.

\subsection{Results}
\label{sec:results}
We now instantiate the gap criterion for data mixture and selection. \Cref{fig:mixture_ratio_evolve} shows close agreement between the analytic predictions and finite-dimensional simulations.

\paragraph{Mixture: Acceleration without advancement.}
Under a uniform target $\widehat{\bw}=\bw_{\mathrm{baseline}}=(0.5,0.5)$, there is no train--target gap. We introduce heterogeneous domain difficulty by giving domain~2 nuisance variance orthogonal to its teacher $\bB_2$, thereby lowering its signal-to-noise ratio without changing its labeling rule (Appendix~\ref{app:mixture}). Allocating more data to this harder domain ($w_1<0.5$) lowers the target loss at small budgets (\Cref{fig:mixture_ratio_evolve}(b)), but the gain vanishes as $\alpha$ grows, and the uniform baseline attains the best endpoint. Domain difficulty improves finite-budget performance without changing the attainable endpoint: it provides acceleration, not advancement.

\paragraph{Mixture: Advancement by matching the target.}
Unlike the finite-budget effect of domain difficulty, the infinite-data endpoint is determined by alignment between the training and target mixtures. At the large budget shown in \Cref{fig:mixture_ratio_evolve}(c), each curve is minimized at $w_1=\widehat w_1$, despite heterogeneous domain difficulty. Appendix~\ref{app:mixture} establishes this result analytically for the large-budget limit ($\alpha \rightarrow \infty$): when the training and target distributions span the same domains, target matching is globally optimal, $\bw^\star=\widehat{\bw}$. Consequently, if $\widehat{\bw}=\bw_{\mathrm{baseline}}$, the baseline already matches the target and no mixture can provide advancement. If $\widehat{\bw}\neq\bw_{\mathrm{baseline}}$, matching corrects the distribution gap and yields $\mathcal E_\infty(\widehat{\bw}) < \mathcal E_\infty(\bw_{\mathrm{baseline}})$, providing genuine advancement.

\paragraph{Selection: A gap creates optimization headroom.}
The same selection policy behaves differently depending on the train--target gap. With $\phi=0$, moderate pruning preserves accuracy, but no policy exceeds full retention over the displayed budgets (\Cref{fig:mixture_ratio_evolve}(e)). With $\phi>0$, target accuracy initially increases as the retained fraction $f$ decreases (\Cref{fig:mixture_ratio_evolve}(f)). Although fewer examples are used, the selection preferentially removes target-inconsistent examples while retaining informative examples near the target boundary. The resulting gain in train--target alignment outweighs the loss in sample size. Beyond the interior optimum $f^\star$, the reduction in the training budget dominates, causing accuracy to plateau (in the large-budget limit, $\alpha \rightarrow \infty$) or decline (at finite budgets). Appendix~\ref{app:imperfect} derives these finite-budget curves and shows that the improvement persists in the large-budget limit.

We now show that the impossibility of advancement without a gap is not an artifact of our perceptron
model, but a fundamental algebraic consequence of a bi-level structure. Curation seeks an allocation $\bw$ to minimize the target loss (outer problem) subject to training convergence (inner problem).

\begin{observation}[Informally]
\label{prop:impossibility_main}
Consider the bi-level curation problem $\min_{\bw}\mathcal L_{\mathrm{tar}}(\btheta^\star(\bw))$, where $\btheta^\star(\bw)\in\arg\min_{\btheta} \mathcal L_{\mathrm{train}}(\btheta;\bw)$. Suppose there is no gap, such that the population training objective~(expected losses) under the uncurated baseline $\bw_{\mathrm{baseline}}$ coincides with the population target objective. Then $\bw_{\mathrm{baseline}}$ is a global minimizer of the outer objective: no allocation attains a lower target endpoint, and curation can at most accelerate.
\end{observation}
The intuition is immediate: without a gap, the target objective is exactly the training objective that the baseline already globally minimizes. We defer the formal proof to Appendix~\ref{app:impossibility}, and next translate this distribution-level condition into an observable diagnostic.

\section[The \texorpdfstring{$\Delta$}{Delta}-Diagnostic]{The $\Delta$-Diagnostic}
\label{sec:delta_test}

The correctable train--target gap is a property of distributions and cannot be observed directly. The $\Delta$-diagnostic provides an observable signature of it from a single uncurated baseline run $\btheta(\bw_{\mathrm{baseline}})$, before any curation is attempted. The idea is simple: if the baseline training distribution matches the target, the baseline should fit the target about as well as it fits the data it was trained on. A clear discrepancy indicates that the baseline optimizes something different from what is evaluated.

Concretely, the baseline draws a fraction $w_k$ of its training data from domain $k\in\{1,\dots,K\}$, and the target consists of components $j\in\{1,\dots,J\}$ with weights $\widehat w_j$. We match each target component $j$ with the set $S_j$ of training domains that correspond to it, such that every training domain belongs to exactly one $S_j$. For a model $\btheta$, the signed discrepancy of group $j$ is
\begin{equation}
d_j(\btheta)=\widehat w_j\,\mathcal L_{\mathrm{tar}}^{j}(\btheta)-\sum_{k\in S_j} w_k\,\mathcal L_{\mathrm{train}}^{k}(\btheta),
\label{eq:signed_disc}
\end{equation}
where $\mathcal L_{\mathrm{train}}^{k}$ is the loss on data from domain $k$ that the baseline has seen during training (Appendix~\ref{app:measuring_delta}), and $\mathcal L_{\mathrm{tar}}^{j}$ is the loss on data of target component $j$. The matching $S$ is fixed by the task, and two cases arise in this paper. When the target components are the training domains themselves (\emph{paired}, $J=K$), each component is matched with its own domain, $S_j=\{j\}$. When the target has no counterpart in the candidate pool, such as a held-out domain (\emph{unpaired}), it forms a single group with all training domains ($J=1$, $S_1=\{1,\dots,K\}$). Selection is the unpaired special case $K=J=1$, with full retention ($f=1$) as the baseline.

\vspace{-3ex}
\paragraph{Entropy correction.} Raw losses also reflect the intrinsic entropy of the data, which no curation policy can change.  For an unpaired case, the entropy difference can inflate or mask the gap we care about. We therefore subtract the same discrepancy measured by a reference model $\btheta_{\mathrm{ref}}$, evaluated on the same documents:
\begin{equation}
\Delta=\sum_{j=1}^{J}\Bigl|\,d_j\bigl(\btheta(\bw_{\mathrm{baseline}})\bigr)\,-\,d_j(\btheta_{\mathrm{ref}})\,\Bigr|.
\label{eq:delta_test}
\end{equation}
The reference only needs to capture the intrinsic difficulty of the data and must not have been trained on the baseline's training data.

Designed as a lightweight operational diagnostic rather than a statistical distance, $\Delta$ indicates the \emph{potential} for advancement. An appreciably nonzero value (more than three standard deviations above a no-gap null; Appendix~\ref{app:measuring_delta}) flags an observable train--target mismatch and hence possible advancement, provided that a feasible curation policy can correct the mismatch. Across both the mixture and selection models, the gain over the baseline increases with $\Delta$ (\Cref{fig:replica-gain-vs-delta}; \Cref{tab:interpolation}). Conversely, $\Delta\approx0$ suggests that curation will only accelerate training, although scalar losses may fail to reveal some distributional gaps, so a near-zero $\Delta$ is not a certificate of no gap. In practice, $\Delta$ can be estimated early in training using a small proxy model---after the gap has become visible but well before full convergence.
\section{Experiments on Real Data}
\label{sec:experiments}

Our theory establishes a correctable train--target gap as a prerequisite for curation to advance a model. We test this in practical LLM training across data mixture and selection: does a nonzero $\Delta$ accompany persistent gains, while $\Delta \approx 0$ suggests limited headroom for consistent improvement over the baseline policy? To ensure a controlled comparison, we evaluate all baselines under a shared training and evaluation protocol and will release all code and configurations. We run each experiment five times independently. Due to limited space, the main text reports means and shows standard-deviation bands where space permits. Complete results are provided in Appendix~\ref{app:real_world_stds}. Across all settings, the diagnostic $\Delta$ is evaluated at exactly 50$\%$ of the training budget.  Earlier checkpoints yield the same diagnosis (\Cref{tab:delta_mixture_ref}).

\subsection{Mixture}
\label{sec:exp-mixture}
We evaluate three scenarios: a matched no-gap setting, an exogenous distribution shift, and an endogenous gap induced by limited domain data. Following \citet{chen2025aioli}, we train 160M GPT-style LMs~\citep{biderman2023pythia} on a 6B sampled version of SlimPajama~\citep{soboleva2023slimpajama} across all seven domains. We compare the Uniform baseline against five representative mixture optimizers: RegMix~\citep{liu2025regmix}, DoGE~\citep{fan2024doge}, Skill-It~\citep{chen2023skillit}, Aioli~\citep{chen2025aioli}, and TANDEM~\citep{wang2025tandem}--spanning mixing-law based, influence-based, and bi-level approaches. Appendix~\ref{app:real_data_details} describes each method and our complete experimental protocol.

\begin{table}[]
\caption{Perplexity in the no-gap~($\Delta=0.076$) mixture setting on SlimPajama. “Average” represents the exponential of the average loss across all domains as in~\citep{fan2024doge}.}
\label{tab:mixture_without_gap}
\centering
\setlength{\tabcolsep}{5pt}
\small
\begin{tabular}{@{}lcccccccc@{}}
\toprule
Methods  & \multicolumn{1}{l}{Arxiv} & \multicolumn{1}{l}{Book} & \multicolumn{1}{l}{C4} & \multicolumn{1}{l}{CommonCrawl} & \multicolumn{1}{l}{Github} & \multicolumn{1}{l}{Stackexchange} & \multicolumn{1}{l}{Wikipedia} & \multicolumn{1}{l}{Average} \\ \midrule                                                                      
Uniform  & 10.02                     & 39.61                    & 49.67                  & 42.70                           & 5.16                       & 12.25                             & 20.00                         & 19.71$\pm_{\text{0.17}}$                       \\
RegMix   &  10.53                    &     41.90                      &    48.67             &     41.65                          &        5.75                     &   11.73                           &      19.86                          &      19.80$\pm_{\text{0.12}}$                        \\
DoGE     &     10.96        &   35.49       &            42.42            &       36.81                          &     6.70                       &       13.97                            &        23.92                       &       20.17$\pm_{\text{0.16}}$                      \\
Skill-It &     10.64                      &     38.32                     &         47.99               &     40.62                            &        5.73                    &          12.07                         &           20.78                    &     19.68$\pm_{\text{0.16}}$                         \\
Aioli    &     10.09                      &        39.86                  &          49.81              &     42.83                            &   5.66                         &    12.36                                &        19.67             &        19.77$\pm_{\text{0.15}}$                     \\
TANDEM   &        10.18                    &  37.89                         &          46.77               &       40.41                           &     5.89                        &    12.61                                &     20.96                           &     19.66$\pm_{\text{0.15}}$                       \\ \bottomrule
\end{tabular}
\end{table}

\begin{figure}[t]
\centering
\begin{minipage}{0.45\linewidth}
    \centering
    \includegraphics[width=1.0\linewidth]{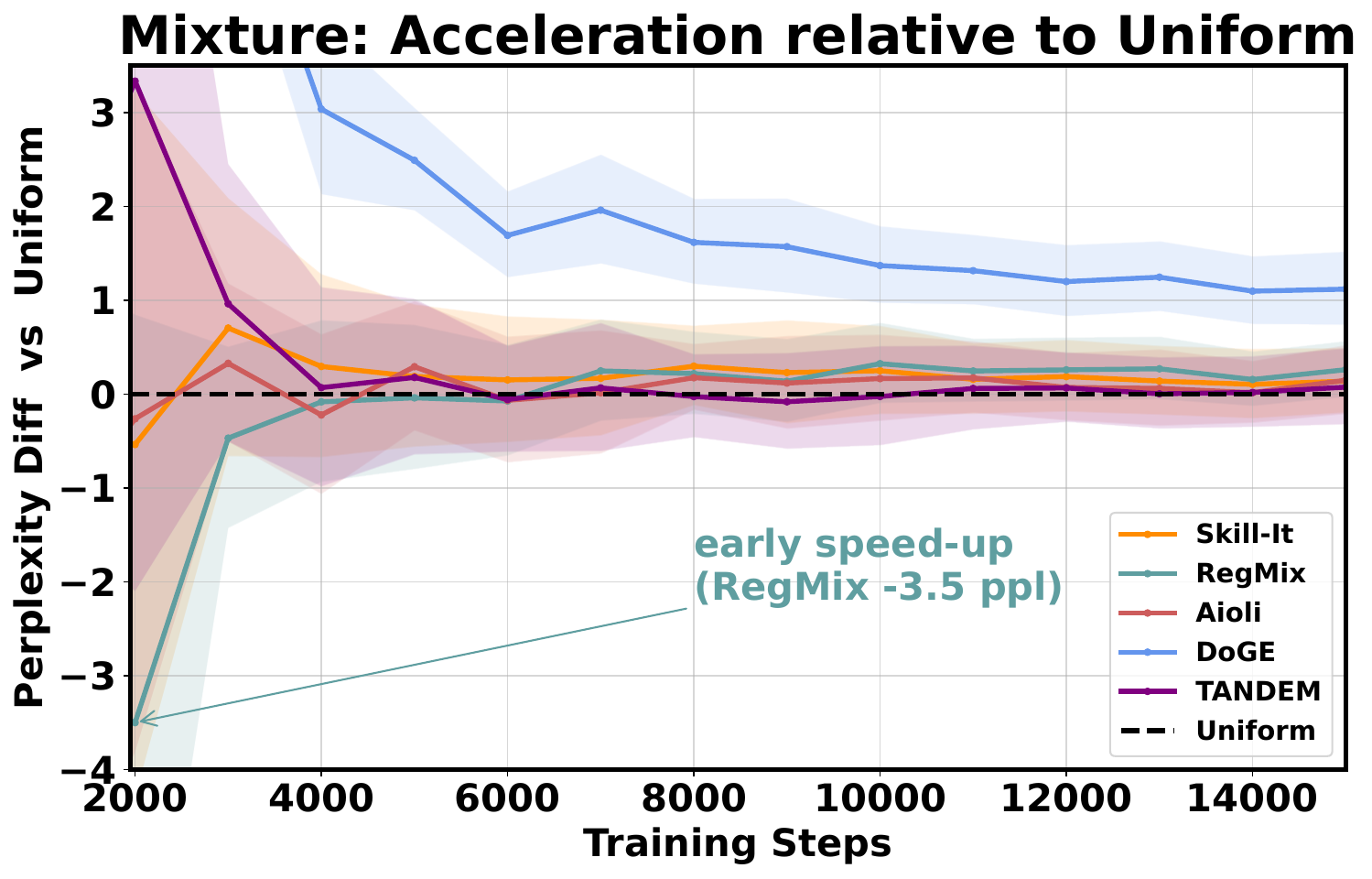}
\end{minipage}
\begin{minipage}{0.45\linewidth}
    \centering
    \includegraphics[width=0.98\linewidth]{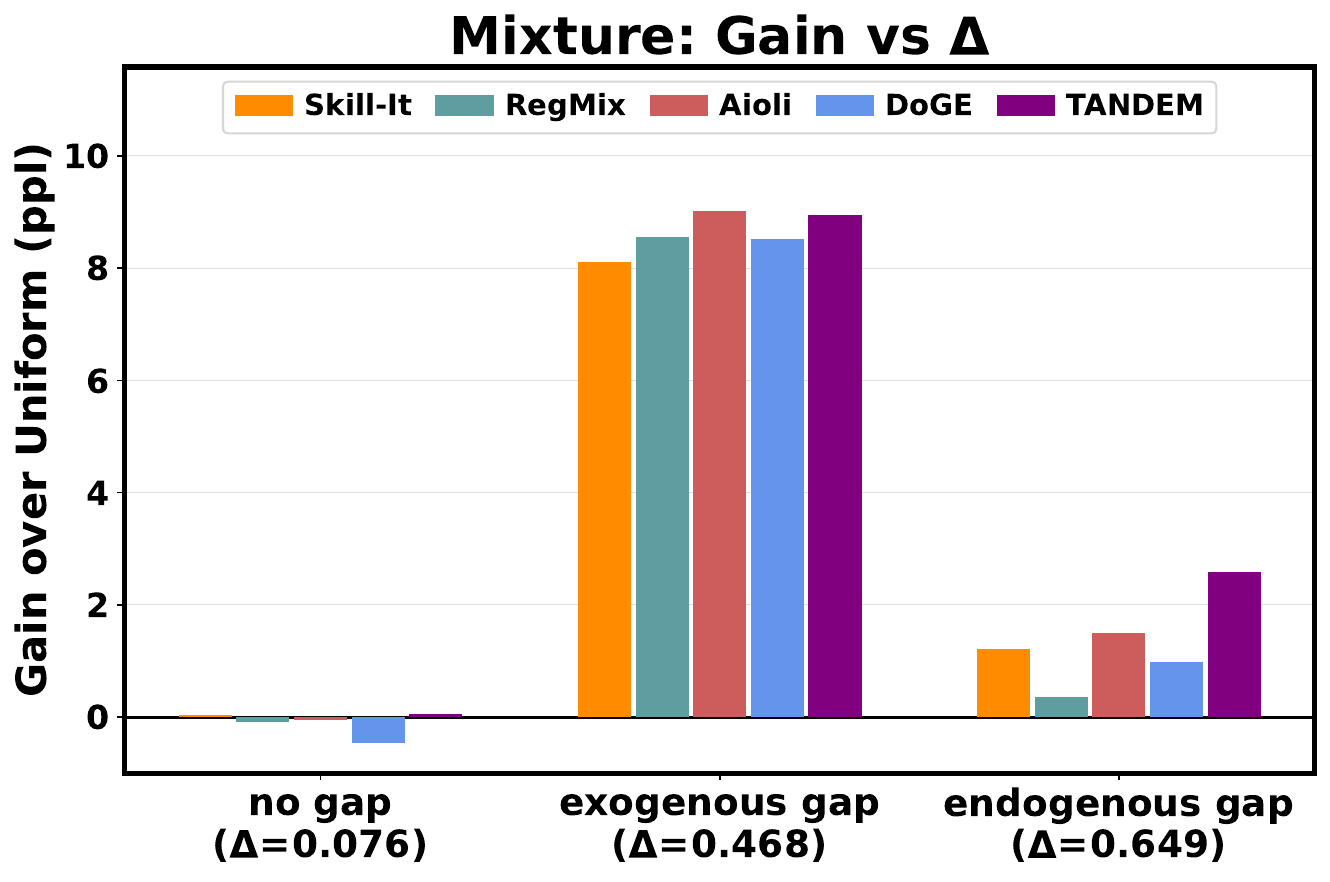}
\end{minipage}
\caption{\emph{Left:} With a negligible gap ($\Delta=0.076$), the perplexity difference relative to Uniform (method $-$ Uniform, lower is better) can be negative early but approaches zero with training. Bands show the standard deviation over five runs. \emph{Right:} Final gain over Uniform (higher is better) is negligible in the no-gap setting and becomes substantial under endogenous and exogenous gaps.}
\label{fig:mixture_results}
\end{figure}

\paragraph{When $\Delta\approx0$, optimized mixtures accelerate but do not advance.} 
We first consider a setting in which the reference training distribution can match the uniformly weighted target. Although SlimPajama-6B is domain-imbalanced, in our training setting, Uniform allocation is feasible without repeatedly sampling a depleted domain, yielding a negligible train--target gap ($\Delta=0.076$). The reference we used for $\Delta$-diagnostic~(\Cref{eq:delta_test}) is the pretrained 160M GPT-style base model. From~\Cref{tab:mixture_without_gap} we see that the final perplexities of all optimized mixtures range from 19.66 to 20.17, within experimental variation of the Uniform baseline at 19.71. Reallocating training exposure across domains does not produce a consistently better result. At intermediate checkpoints, RegMix reduces perplexity by 3.5, but this advantage vanishes as the training proceeds~(\Cref{fig:mixture_results}, left). This no-gap setting exhibits acceleration rather than advancement, supporting our theoretical prediction.

\begin{table}[]
\caption{Final target perplexity under endogenous ($\Delta=0.649$) and exogenous ($\Delta=0.468$) mixture gaps. The endogenous block reports all target domains and their "Average". The exogenous block reports the perplexity of the held-out CommonCrawl target (CC).}
\label{tab:mixture_with_gap}
\centering
\setlength{\tabcolsep}{4pt}
\small
\begin{tabular}{@{}lcccccccc|c@{}}
\toprule
\multicolumn{9}{c|}{Endogenous Gap}                 & \multicolumn{1}{c}{Exogenous Gap} \\ \midrule
Methods  & \multicolumn{1}{l}{Arxiv} & \multicolumn{1}{l}{Book} & \multicolumn{1}{l}{C4} & \multicolumn{1}{l}{CC.} & \multicolumn{1}{l}{Github} & \multicolumn{1}{l}{Stackexchange} & \multicolumn{1}{l}{Wikipedia} & \multicolumn{1}{l|}{Average} & CC.                \\ \midrule
Uniform  & 13.14 & 57.44 & 65.57 & 57.93 & 8.21  & 19.59 & 33.90 & 28.60    &  49.19                \\
RegMix  & 13.33 & 54.12 & 63.74 & 56.63 & 8.23 & 20.03 & 33.34 & 28.24 &            40.63                \\
DoGE   & 12.95 & 54.82 & 62.84 & 55.52 & 8.13  & 18.61 & 32.72 & 27.62 &            40.67              \\
Skill-It & 14.54 & 44.24 & 48.99 & 42.10 & 8.33  & 21.10 & 49.46 & 27.38  &       41.08                \\
Aioli & 12.80 & 51.01 & 61.13 & 53.51 & 8.14  & 18.97 & 32.63 & 27.11 &          40.17                \\
TANDEM        & 12.60 & 49.76 & 54.19 & 47.24 & 8.35  & 17.98 & 33.51 & 26.02      &        40.24                    \\ \bottomrule
\end{tabular}
\end{table}

\paragraph{When $\Delta>0$ under distribution shift, optimized mixtures can
advance.}
An exogenous gap arises when the candidate pool and target follow different distributions. We instantiate this setting by holding out CommonCrawl as the target while training on the remaining six domains, yielding $\Delta\approx0.468$, far above the no-gap value of $0.076$. Although CommonCrawl is absent from the training pool, the related C4 domain provides a feasible source for partially correcting the gap. As shown in \Cref{tab:mixture_with_gap}, every optimized mixture substantially outperforms Uniform, reducing target perplexity from 49.19 to 40.2--41.1. By shifting training exposure toward domains better aligned with the held-out target, curation retains a large gain through the largest feasible training budget, with no sign of the advantage shrinking. In contrast to the no-gap gains that vanish with training, this persistence provides empirical evidence of advancement rather than transient acceleration.

\paragraph{When $\Delta>0$ under limited data, optimized mixtures can
advance.}
Repeated sampling of data-scarce domains produces an endogenous train--target
gap: overfitting raises their target losses relative to their training losses,
thereby increasing $\Delta$. We instantiate this setting by restricting the
SlimPajama pool to 300M tokens while preserving its original domain
distribution. Under the fixed training budget, small domains such as ArXiv and
Wikipedia are exhausted early and repeatedly sampled, yielding
$\Delta=0.649$. As shown in \Cref{tab:mixture_with_gap}, all optimized mixtures
outperform Uniform at the largest tested budget by reducing exposure to these
overfit domains. Together with the exogenous-gap result, this provides further evidence that a correctable gap can enable advancement regardless
of whether it arises from distribution shift or limited data
(right panel of \Cref{fig:mixture_results}).

\begin{wraptable}{r}{0.56\linewidth}
\caption{Scale study comparing Uniform and TANDEM. Gain is Uniform $-$ TANDEM.}
\label{tab:mixture_scale}
\centering
\small
\setlength{\tabcolsep}{3pt}
\begin{tabular}{@{}llccc@{}}
\toprule
Scenario & Model & Uniform & TANDEM & Gain \\ \midrule
\multirow{2}{*}{\shortstack[l]{No gap\\$\Delta=0.076$}}
 & 160M & 19.71 & 19.66 & $+0.05$ \\
 & 1.1B & 17.16 & 17.05 & $+0.11$ \\ \midrule
\multirow{2}{*}{\shortstack[l]{Endogenous gap\\$\Delta=0.649$}}
 & 160M & 28.60 & 26.02 & $+2.58$ \\
 & 1.1B & 28.14 & 25.85 & $+2.29$ \\ \midrule
\multirow{2}{*}{\shortstack[l]{Exogenous gap\\$\Delta=0.468$}}
 & 160M & 49.19 & 40.24 & $+8.95$ \\
 & 1.1B & 40.03 & 34.72 & $+5.31$ \\ \bottomrule
\end{tabular}
\end{wraptable}

\paragraph{The $\Delta$-diagnostic remains predictive across model scales.}
We repeat all three mixture settings at 1.1B parameters and compare TANDEM with Uniform (\Cref{tab:mixture_scale}). At 160M and 1.1B, respectively, the no-gap gains remain negligible ($+0.05$ and $+0.11$ perplexity), whereas the exogenous-gap gains remain substantial ($+8.95$ and $+5.31$), as do the endogenous-gap gains ($+2.58$ and $+2.29$). Thus, the diagnosis obtained at 160M transfers to 1.1B, supporting the use of a compute-efficient proxy model to screen for advancement potential before full-scale training
(\Cref{sec:delta_test}).

\subsection{Selection}
\label{sec:exp-selection}

To test whether persistent selection gains require a correctable mismatch, we use a paired design with no-gap and exogenous-gap scenarios. The two scenarios share the same candidate pool, model, selection methods, and training protocol. Only the target distribution used for selection and evaluation changes. This design isolates the role of the gap rather than merely comparing selection algorithms. We use a 5B-token OpenWebMath candidate pool~\citep{paster2024openwebmath} and TinyLlama-1.1B~\citep{zhang2024tinyllama}\footnote{We choose TinyLlama over a heavily optimized SOTA model to prevent saturated capabilities from masking the intrinsic efficacy of data selection.}. We report held-out target loss and accuracy on eight benchmarks spanning arithmetic, competition mathematics, word problems, table-based reasoning, and STEM knowledge. Baselines include full-pool continuous pretraining (Base-CT), random selection, three target-agnostic methods (EL2N~\citep{mansheej21datadiet}, InfoBatch~\citep{qin2024infobatch}, and Token-Cleaning~\citep{pang2025tokencleaning}), and two target-aware methods (RHO-1~\citep{lin2024rho1} and BLADE~\citep{wang2026blade}). Appendix~\ref{app:real_data_details} provides the complete benchmark list, method descriptions, and experimental protocol.

\begin{table}[]
\caption{Data-selection results for an in-distribution OpenWebMath target (no gap, $\Delta=0.016$, upper) and a shifted MetaMath--Mammoth target (with gap, $\Delta=0.158$). We report accuracy on eight math benchmarks, their average, and held-out target loss. All methods operate on the same 5B-token candidate pool with a retained fraction of $f=0.6$. Target-aware RHO-1 and BLADE improve substantially only when a gap is present.}
\label{tab:sota_math}
\setlength\tabcolsep{2.8pt}
\centering
\small
\resizebox{\linewidth}{!}{%
\begin{tabular}{@{}lccccccccccc@{}}
\toprule
                & Retained & \multirow{2}{*}{GSM8K} & \multirow{2}{*}{MATH} & \multirow{2}{*}{SVAMP} & \multirow{2}{*}{ASDiv} & \multirow{2}{*}{MAWPS} & \multirow{2}{*}{TAB} & \multirow{2}{*}{MQA} & MMLU- & \multirow{2}{*}{Avg.} & Target \\
                & Toks. &                         &                       &                        &                        &                        &                      &                      & STEM  &                     & Loss   \\ \midrule
\multicolumn{12}{c}{Selection without Gap}                                                                                                                                                                                                                                                                  \\ \midrule
Base            & -                               & 2.9                      & 3.2                      & 11.0                      & 18.1                      & 20.4                      & 12.5                    & 14.6                    & 16.1                          &      12.4 & 1.97              \\
Base-CT         & 5B                             & 4.3  & 4.3    & 17.1   & 29.4   & 39.4   & 14.5  & 8.3   & 20.4        & 17.21 & 1.70                    \\
Random          & 3B                              & 4.0    & 3.9    & 15.8   & 29.1   & 37.2   & 13.8  & 9.4   & 17.8        & 16.38 & 1.73                    \\
EL2N            & 3B  & 3.0    & 3.2    & 15.5   & 25.0   & 31.2   & 14.4  & 8.9   & 18.3        & 14.94 & 1.76                    \\
InfoBatch       & 3B  & 3.8    & 3.3    & 15.8   & 25.8   & 32.9   & 12.5  & 7.1   & 16.5        & 14.71 & 1.76                    \\
Token Cleaning  & 3B                        & 4.2    & 3.0    & 15.4   & 24.4   & 31.7   & 11.7  & 10.2   & 14.1        & 14.30 & 1.79                     \\
RHO-1 & 3B  & 3.2    & 3.1    & 15.3   & 24.6   & 31.4   & 14.5  & 7.1   & 16.7        & 14.49 & 1.78                    \\
BLADE           & 3B     & 4.5    & 3.7    & 18.1   & 29.9   & 37.9   & 16.5  & 10.4   & 19.3        & 17.54 & 1.73                    \\
\midrule \midrule
\multicolumn{12}{c}{Selection with Gap}                                                                                                                                                                                                                                                                     \\ \midrule
Base            & -                               & 2.9                      & 3.2                      & 11.0                      & 18.1                      & 20.4                      & 12.5                    & 14.6                    & 16.1                          &      12.40 & 1.26              \\
Base-CT         & 5B                             & 4.3    & 4.3    & 17.1   & 29.4   & 39.4   & 14.5  & 8.3   & 20.4        & 17.21 & 1.14                   \\
Random          & 3B                              & 4.0    & 3.9    & 15.8   & 29.1   & 37.2   & 13.8  & 9.4   & 17.8        & 16.38 & 1.13                    \\
EL2N            & 3B  & 3.0    & 3.2    & 15.5   & 25.0   & 31.2   & 14.4  & 8.9   & 18.3        & 14.94 & 1.19                    \\
InfoBatch       & 3B  & 3.8    & 3.3    & 15.8   & 25.8   & 32.9   & 12.5  & 7.1   & 16.5    & 14.71 & 1.16                    \\
Token Cleaning  & 3B                        & 4.2    & 3.0    & 15.4   & 24.4   & 31.7   & 11.7  & 10.2   & 14.1        & 14.30 & 1.22                    \\
RHO-1 & 3B    & 15.1   & 5.9    & 32.6   & 46.9   & 63.8   & 18.7  & 15.1   & 18.4        & 27.06 & 0.59                     \\
BLADE           & 3B     & 21.2    & 8.1    & 37.0   & 51.4   & 70.2   & 21.2  & 20.6   & 20.8        & 31.31 & 0.45                    \\
\bottomrule
\end{tabular}
}
\end{table}

\begin{figure}[t]
\centering
\begin{minipage}{0.47\linewidth}
    \centering
    \includegraphics[width=0.98\linewidth]{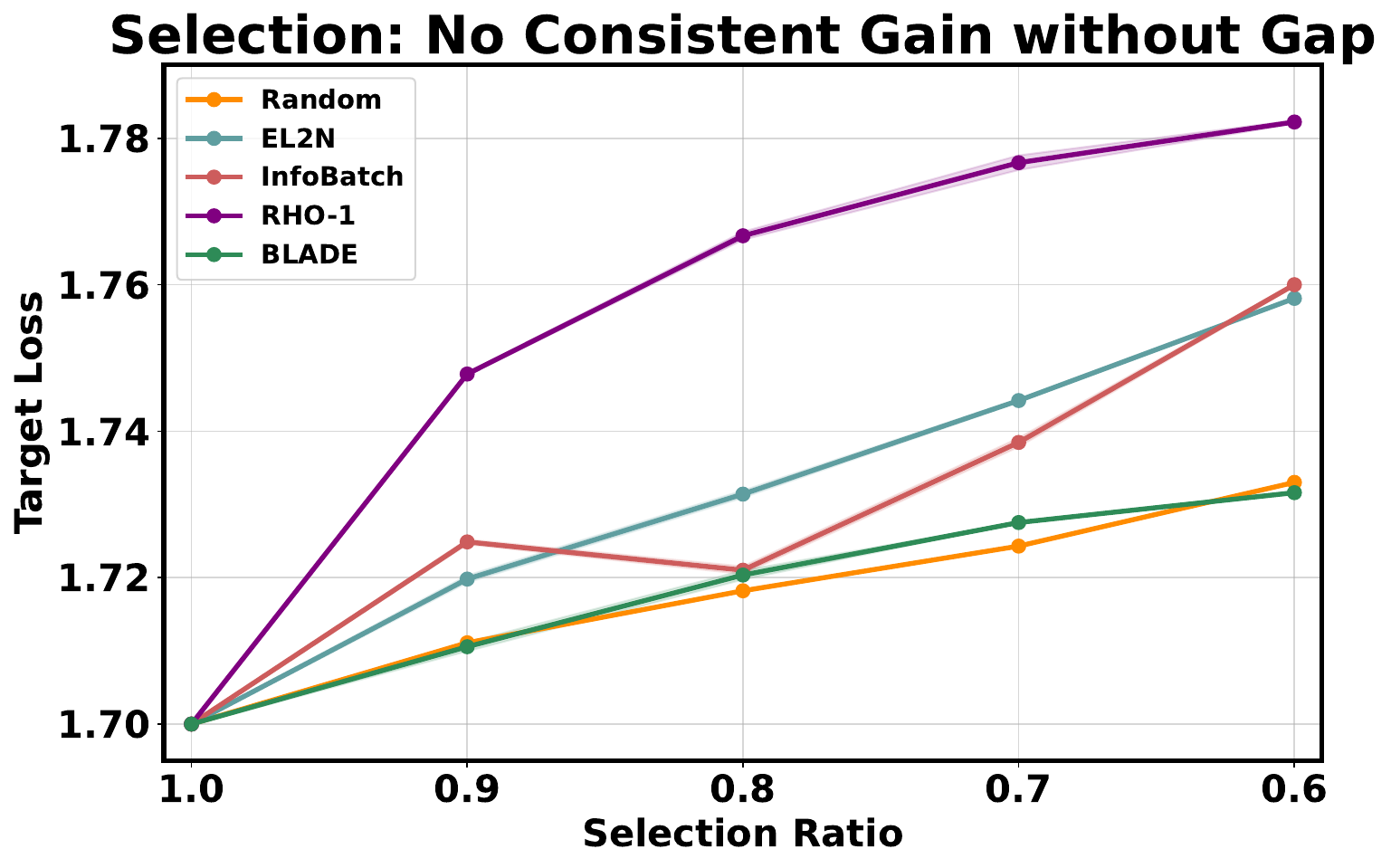}
\end{minipage}
\begin{minipage}{0.45\linewidth}
    \centering
    \includegraphics[width=1.0\linewidth]{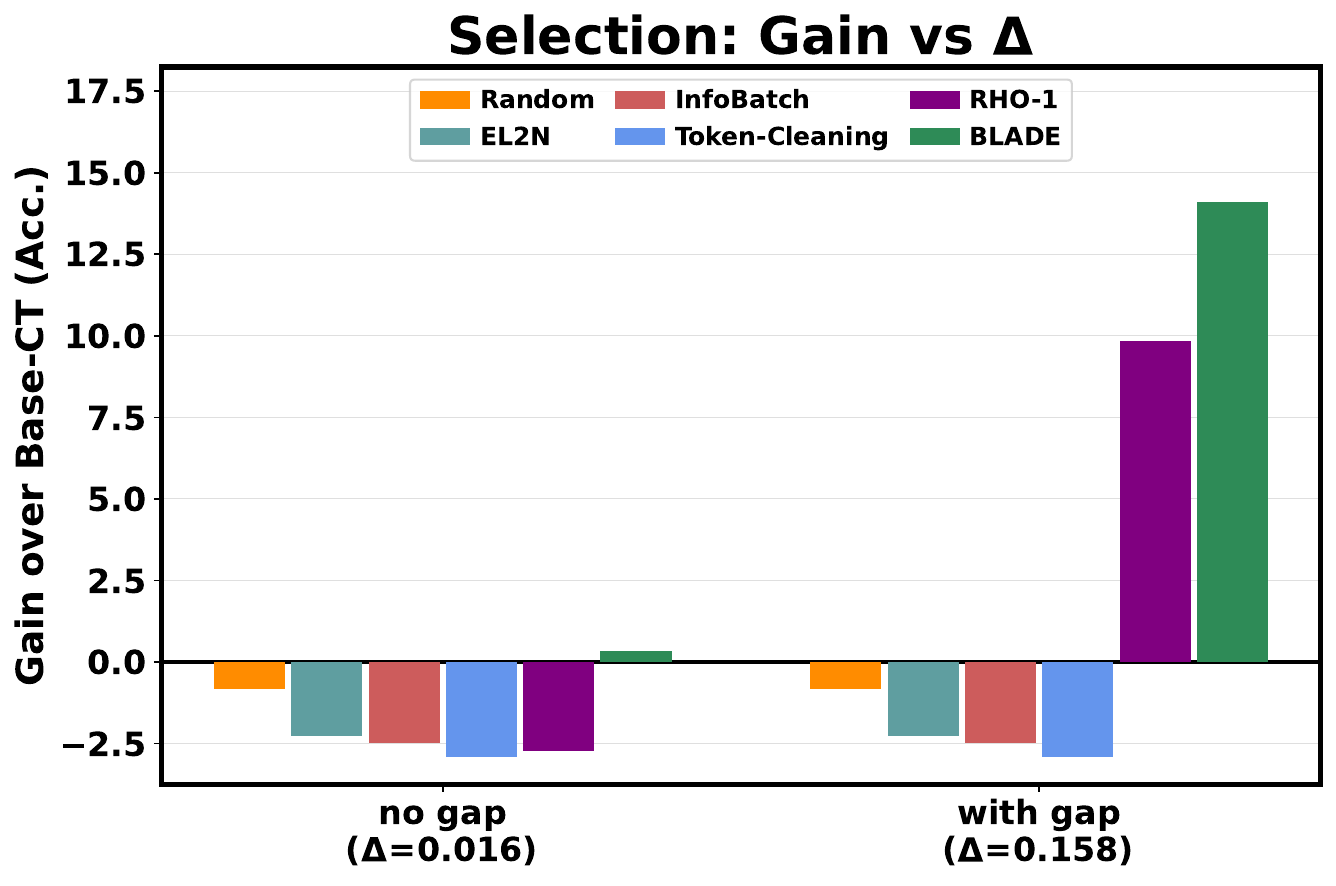}
\end{minipage}
\caption{\emph{Left:} Held-out target loss across retained fractions \(f\in[0.6,1.0]\) for an in-distribution OpenWebMath target (\(\Delta=0.016\)). Target loss directly evaluates the distribution used to define $\Delta$. No selector consistently improves over Base-CT and Random, and more aggressive pruning generally increases loss. \emph{Right:} Average accuracy change relative to Base-CT in the no-gap and shifted-target settings. The comparison shows little optimization headroom in the no-gap setting, whereas substantial positive gains become possible after the target shift ($\Delta=0.158$).}
\label{fig:selection_results}
\end{figure}

\paragraph{When $\Delta\approx0$, selection provides no consistent
advancement.}
Using an additional 0.3B-token sample from OpenWebMath as the target yields a negligible train--target gap ($\Delta=0.016$). The reference we used for $\Delta$-diagnostic~(\Cref{eq:delta_test}) is the pretrained TinyLlama-1.1B base model. With little discrepancy to correct, neither target-aware nor target-agnostic methods consistently outperform the Base-CT (17.21) and Random (16.38) baselines: their average scores range from 14.3 to 17.5, and their held-out target losses show no consistent improvement (\Cref{tab:sota_math}, upper). Notably, RHO-1 and BLADE receive target data but still fail to provide consistent gains, showing that target access alone does not guarantee advancement. \Cref{fig:selection_results} (Left) further examines the effect of the retained fraction. We report held-out target loss rather than downstream accuracy because it is measured on the target distribution used to define $\Delta$, and thus more directly exposes the train--target gap. Varying $f$ from $1.0$ to $0.6$ yields the same conclusion: full retention attains the lowest loss, no selection method consistently improves over Random, and target loss generally rises as selection becomes more aggressive. This retention sweep is consistent with the limited optimization headroom indicated by the negligible gap.

\paragraph{When $\Delta>0$, selection can advance.}
When the target departs from the candidate pool, an exogenous gap emerges. Following \citet{lin2024rho1}, we use a blend of MetaMath~\citep{yu2024metamath} and Mammoth~\citep{yu22024mammoth} as the target while retaining OpenWebMath as the candidate pool,\footnote{This open-source blend is a subset of RHO-1's target data. The remainder has not been released.} yielding a substantial mismatch ($\Delta=0.158$). As shown in \Cref{tab:sota_math} (bottom), target-aware selection produces large gains: RHO-1 and BLADE increase the Base-CT average score from 17.21 to 27.06 and 31.31, respectively, whereas the target-agnostic methods EL2N, InfoBatch, and Token-Cleaning do not. Held-out target losses show the same pattern. Together with the no-gap control, this contrast provides empirical evidence that target-aware selection can achieve advancement by correcting a genuine train--target gap.

\begin{wraptable}{r}{0.5\linewidth}
\vspace{-2ex}
\caption{Target interpolation. A fraction $p$ of the target comes from MetaMath--Mammoth, the rest from OpenWebMath. Gain is the BLADE average score minus that of Base-CT.}
\label{tab:interpolation}
\centering
\small
\setlength{\tabcolsep}{3pt}
\begingroup\color{black}
\begin{tabular}{@{}lcccccc@{}}
\toprule
$p$ & $0$ & $0.2$ & $0.4$ & $0.6$ & $0.8$ & $1.0$ \\ \midrule
$\Delta$ & $0.016$ & $0.041$ & $0.069$ & $0.101$ & $0.129$ & $0.158$ \\
Gain & $+0.3$ & $+9.6$ & $+10.4$ & $+12.0$ & $+12.3$ & $+14.1$ \\ \bottomrule
\end{tabular}
\endgroup
\vspace{-2ex}
\end{wraptable}

\paragraph{The gain grows with $\Delta$.}
To test whether $\Delta$ indicates the size of the attainable gain and not only its presence, we interpolate between the two targets, drawing a fraction $p$ of target tokens from MetaMath--Mammoth and the rest from OpenWebMath, and rerun BLADE for each $p$. As the target shifts, $\Delta$ increases monotonically from $0.016$ to $0.158$, and the BLADE gain over Base-CT increases with it from $+0.3$ to $+14.1$ (\Cref{tab:interpolation}; Spearman $\rho=1$).

\section{Related Work}
\label{sec:related}

\paragraph{Data mixture optimization.}
Beyond hand-tuned compositions, mixture methods learn domain weights via mixing laws~\citep{liu2025regmix, chen2023skillit, chen2025aioli} or bi-level optimization~\citep{fan2024doge, wang2025tandem}. Empirical comparisons have also found Uniform to remain a competitive baseline in several settings~\citep{chen2025aioli}. Our analysis offers one explanation for this pattern: when the baseline training distribution is aligned with the target, Uniform is endpoint-optimal, and a finite-budget advantage reflects acceleration.

\paragraph{Data selection and quality scores.}
Data selection methods generally fall into \emph{target-agnostic} and \emph{target-aware} paradigms. \emph{Target-agnostic} scores rank examples by intrinsic difficulty or quality signals~\citep{mansheej21datadiet, wang2023dynaMS, qin2024infobatch, pang2025tokencleaning}. Nevertheless, they have a mixed record, with reports that random selection remains highly competitive in the data-rich regime~\citep{xia2024random}. Through our lens, such difficulty proxies accelerate learning but do not advance it when no train--target gap exists. Conversely, \emph{target-aware} methods can genuinely advance the model by aligning with a target. This alignment can be \emph{implicit}, such as DCLM~\citep{li2024dclm} and DEITA~\citep{liu2024deita} using language models whose priors encode target distributions, or \emph{explicit}, where candidates are scored directly against a defined target~\citep{lin2024rho1,yu2024mates,mizrahi2025target,wang2026blade}. These methods provide advancement by successfully reducing a genuine gap.

\paragraph{Solvable models of data curation.}
Another line of research explores data curation through solvable statistical-mechanics models, using the teacher--student perceptron and replica method~\citep{engel2001statisticalml, bruno2021learningcurve, sorscher2022beyondsl}. Closest to our setup, \citet{dohmatob2025less} derive exact scaling laws for data selection under distribution shift to characterize when "less is more". While we share this solvable-model viewpoint, our approach differs in key aspects. First, we directly formulate the classification objective rather than relying on a mean squared error (MSE) surrogate. Second, we explicitly formalize the distinction between \emph{acceleration} and \emph{advancement}. Finally, we translate these insights and advancement prerequisites into an observable empirical diagnostic (the $\Delta$-diagnostic), unifying both mixture and selection under a single framework, and validated on real LLMs.
\section{Conclusion}
\label{sec:conclusion}

We show that the value of data curation depends on a distinction: whether a strategy advances attainable target performance or accelerates training. Advancement requires a correctable mismatch between the training baseline and the target. When the baseline training and target objectives coincide, the uncurated baseline is endpoint-optimal, so the benefit of curation can only be acceleration. This helps explain why uniform mixtures and random selection remain difficult to beat in data-rich, in-distribution regimes. Practically, our lightweight $\Delta$-diagnostic diagnoses advancement potential from a baseline run before expensive optimization. A near-zero $\Delta$ suggests evaluating curation by the compute it saves rather than by its final performance, whereas an appreciably nonzero $\Delta$ signals potential headroom for target-aware optimization. Ultimately, our findings recenter curation on the task definition itself. Given the asymmetry between massive candidate training pools and compact, controllable target sets, constructing diverse, high-fidelity targets is an important design choice in its own right, not only an evaluation concern.

\bibliography{iclr2027_conference}

@inproceedings{yu2024metamath,
title={MetaMath: Bootstrap Your Own Mathematical Questions for Large Language Models},
author={Longhui Yu and Weisen Jiang and Han Shi and Jincheng YU and Zhengying Liu and Yu Zhang and James Kwok and Zhenguo Li and Adrian Weller and Weiyang Liu},
booktitle={International Conference on Learning Representations},
year={2024}
}

@inproceedings{sorscher2022beyondsl,
title={Beyond neural scaling laws: beating power law scaling via data pruning},
author={Ben Sorscher and Robert Geirhos and Shashank Shekhar and Surya Ganguli and Ari S. Morcos},
booktitle={Advances in Neural Information Processing Systems},
year={2022},
}

@inproceedings{engel2001statisticalml,
title={Statistical Mechanics of Learning},
author={A. Engel and C. Van den Broeck},
booktitle={Cambridge University Press},
year={2022},
}

@inproceedings{bruno2021learningcurve,
title={Learning curves of generic features maps for realistic datasets with a teacher-student model},
author={Bruno Loureiro and Cedric Gerbelot and Hugo Cui and Sebastian Goldt and Florent Krzakala and Marc Mezard and Lenka Zdeborova},
booktitle={Advances in Neural Information Processing Systems},
year={2021},
}

@inproceedings{xie2023doremi,
title={DoReMi: Optimizing Data Mixtures Speeds Up Language Model Pretraining},
author={Sang Michael Xie and Hieu Pham and Xuanyi Dong and Nan Du and Hanxiao Liu and Yifeng Lu and Percy Liang and Quoc V Le and Tengyu Ma and Adams Wei Yu},
booktitle={Neural Information Processing Systems},
year={2023},
}

@inproceedings{liu2025regmix,
title={RegMix: Data Mixture as Regression for Language Model Pre-training},
author={Qian Liu and Xiaosen Zheng and Niklas Muennighoff and Guangtao Zeng and Longxu Dou and Tianyu Pang and Jing Jiang and Min Lin},
booktitle={International Conference on Learning Representations},
year={2025}
}

@inproceedings{fan2024doge,
title={DOGE: Domain Reweighting with Generalization Estimation},
author={Simin Fan and Matteo Pagliardini and Martin Jaggi},
booktitle={International Conference on Machine Learning},
year={2024},
}

@inproceedings{chen2023skillit,
title={Skill-It! A Data-Driven Skills Framework for Understanding and Training Language Models},
author={Mayee F. Chen and Nicholas Roberts and Kush Bhatia and Jue Wang and Ce Zhang and Frederic Sala and Christopher Ré},
booktitle={Neural Information Processing Systems},
year={2023}
}

@inproceedings{chen2025aioli,
title={Aioli: A unified optimization framework for language model data mixing},
author={Mayee F. Chen and Michael Y. Hu and Nicholas Lourie and Kyunghyun Cho and Christopher Ré},
booktitle={International Conference on Learning Representations},
year={2025},
}

@inproceedings{wang2025tandem,
title={TANDEM: Bi-Level Data Mixture Optimization with Twin Networks},
author={Jiaxing Wang and Deping Xiang and Jin Xu and Mingyang Yi and Guoqiang Gong and Zicheng Zhang and Haoran Li and Pengzhang Liu and Zhen Chen and Ke Zhang and Ju Fan and Qixiang Jiang},
booktitle={Neural Information Processing Systems},
year={2025},
}

@InProceedings{biderman2023pythia,
  title = 	 {Pythia: A suite for analyzing large language models across training and scaling},
  author =       {Stella Biderman and Hailey Schoelkopf and Quentin Gregory Anthony and Herbie Bradley and Kyle O’Brien and Eric Hallahan and Mohammad Aflah Khan and Shivanshu Purohit and USVSN Sai Prashanth and Edward Raff eta.al.},
  booktitle = 	 {International Conference on Machine Learning},
  year = 	 {2023}
}

@article{soboleva2023slimpajama,
  author       = {Daria Soboleva and Faisal Al-Khateeb and Joel Hestness and Nolan Dey
Opentensor: Robert Myers, Jacob Robert Steeves},
  title        = {SlimPajama: A 627B token, cleaned and deduplicated version of RedPajam},
  journal = {https://www.cerebras.ai/blog/slimpajama-a-627b-token-cleaned-and-deduplicated-version-of-redpajama},
  year         = {2023}
}

@article{xia2024random,
  author    = {Tingyu Xia and Bowen Yu and Kai Dang and An Yang and Yuan Wu and Yuan Tian and Yi Chang and Junyang Lin},
  title     = {Rethinking Data Selection at Scale: Random Selection is Almost All You Need},
  journal   = {arXiv preprint},
  volume    = {arXiv:2410.09335},
  year      = {2024}
}

@inproceedings{lin2024rho1,
title={RHO-1: Not All Tokens Are What You Need},
author={Zhenghao Lin and Zhibin Gou and Yeyun Gong and Xiao Liu and Yelong Shen and Ruochen Xu and Chen Lin and Yujiu Yang and Jian Jiao and Nan Duan and Weizhu Chen},
booktitle={Neural Information Processing Systems},
year={2024},
}

@inproceedings{yu2024mates,
title={MATES: Model-Aware Data Selection for Efficient Pretraining with Data Influence Models},
author={Zichun Yu and Spandan Das and Chenyan Xiong},
booktitle={Neural Information Processing Systems},
year={2024},
}

@inproceedings{pang2025tokencleaning,
  title = 	{Token Cleaning: Fine-Grained Data Selection for LLM Supervised Fine-Tuning},
  author =    {Jinlong Pang and Na Di and Zhaowei Zhu and Jiaheng Wei and Hao Cheng and Chen Qian and Yang Liu},
  booktitle = {International Conference on Machine Learning},
  year = 	 {2025}
}

@article{wang2026blade,
  title={BLADE: Scalable Bi-level Adaptive Data Selection for LLM Training},
  author={Jiaxing Wang and Deping Xiang and Jin Xu and Zirui Liu and Zicheng Zhang and Guoqiang Gong and Jun Fang and Chao Liu and Pengzhang Liu and Tongxuan Liu and Ke Zhang and Qixia Jiang},
  journal={arXiv preprint arXiv:2606.18650},
  year={2026}
}

@inproceedings{liu2024deita,
  title={What makes good data for alignment? a comprehensive study of automatic data selection in instruction tuning},
  author={Liu, Wei and Zeng, Weihao and He, Keqing and Jiang, Yong and He, Junxian},
  booktitle={International Conference on Learning Representations},
  year={2024}
}

@inproceedings{yu22024mammoth,
title={Mammoth: Building Math Generalist Models Through Hybrid Instruction Tuning},
author={Xiang Yue and Xingwei Qu and Ge Zhang and Yao Fu and Wenhao Huang and Huan Sun and Yu Su and and Wenhu Chen},
booktitle={International Conference on Learning Representations},
year={2024},
}

@article{cobbe2021gsm8k,
  author       = {Karl Cobbe and Vineet Kosaraju and Mohammad Bavarian and Mark Chen and Heewoo Jun and Lukasz Kaiser and Matthias Plappert and Jerry Tworek and Jacob Hilton and Reiichiro Nakano and Christopher Hesse and John Schulman},
  title        = {Training verifiers to solve math word problems},
  journal      = {arXiv preprint},
  volume       = {arXiv:2110.14168},
  year         = {2021}
}

@inproceedings{hendrycks2021math,
title={Measuring mathematical problem solving with the math dataset},
author={Dan Hendrycks and Collin Burns and Saurav Kadavath and Akul Arora and Steven Basart and Eric Tang and Dawn Song and Jacob Steinhardt},
booktitle={Neural Information Processing Systems},
year={2021},
}

@inproceedings{patel2021svamp,
title={Are NLP models really able to solve simple math word problems?},
author={Arkil Patel and Satwik Bhattamishra and Navin Goyal},
booktitle={Conference of the North American Chapter of the Association for Computational Linguistics},
year={2021},
}

@InProceedings{miao2020asdiv,
  title = 	 {A diverse corpus for evaluating and developing English math word problem solvers},
  author =       {Shenyun Miao and Chaochun Liang and Keh-Yih Su},
  booktitle = 	 {Association for Computational Linguistics},
  year = 	 {2020}
}

@InProceedings{kedziorski2016mawps,
  title = 	 {MAWPS: A math word problem repository},
  author =       {Rik Koncel Kedziorski and Subhro Roy and Aida Amini and Nate Kushman and Hannaneh Hajishirzi},
  booktitle = 	 {Conference of the North American Chapter of the Association for Computational Linguistics},
  year = 	 {2016}
}

@inproceedings{lu2023tab,
title={Dynamic prompt learning via policy gradient for semi-structured mathematical reasoning},
author={Pan Lu and Liang Qiu and Kai-Wei Chang and Ying Nian Wu and Song-Chun Zhu and Tanmay Rajpurohit and Peter Clark and Ashwin Kalyan},
booktitle={International Conference on Learning Representations},
year={2023},
}

@article{amini2019mathqa,
  author       = {Aida Amini and Saadia Gabriel and Peter Lin and Rik Koncel-Kedziorski and Yejin Choi and Hannaneh Hajishirzi},
  title        = {MATHQA: Towards Interpretable Math Word Problem Solving with Operation-based Formalisms},
  journal      = {arXiv preprint},
  volume       = {arXiv:1905.13319},
  year         = {2019}
}

@article{hendrycks2020mmlu,
  author       = {Dan Hendrycks and Collin Burns and Steven Basart and Andy Zou and Mantas Mazeika and Dawn Song and Jacob Steinhardt},
  title        = {Measuring Massive Multitask Language Understanding},
  journal      = {arXiv preprint},
  volume       = {arXiv:2009.03300},
  year         = {2020}
}

@inproceedings{paster2024openwebmath,
title={OpenWebMath: An Open Dataset of High-Quality Mathematical Web Text.},
author={Keiran Paster and Marco Dos Santos and Zhangir Azerbayev and Jimmy Ba},
booktitle={International Conference on Learning Representations},
year={2024},
}

@article{zhang2024tinyllama,
  author       = {Peiyuan Zhang and Guangtao Zeng and Tianduo Wang and Wei Lu},
  title        = {Tinyllama: An open-source small language model},
  journal      = {arXiv preprint},
  volume       = {arXiv:2401.02385},
  year         = {2024}
}

@inproceedings{mansheej21datadiet,
 author = {Paul, Mansheej and Ganguli, Surya and Dziugaite, Gintare Karolina},
 booktitle = {Advances in Neural Information Processing Systems},
 editor = {M. Ranzato and A. Beygelzimer and Y. Dauphin and P.S. Liang and J. Wortman Vaughan},
 pages = {20596--20607},
 publisher = {Curran Associates, Inc.},
 title = {Deep Learning on a Data Diet: Finding Important Examples Early in Training},
 volume = {34},
 year = {2021}
}

@inproceedings{qin2024infobatch,
  title={InfoBatch: Lossless Training Speed Up by Unbiased Dynamic Data Pruning},
  author={Qin, Ziheng and Wang, Kai and Zheng, Zangwei and Gu, Jianyang and Peng, Xiangyu and Zhaopan Xu and Zhou, Daquan and Lei Shang and Baigui Sun and Xuansong Xie and You, Yang},
  booktitle={International Conference on Learning Representations},
  year={2024},
}

@inproceedings{li2024dclm,
title={DataComp-LM: In search of the next generation of training sets for language models},
author={Jeffrey Li and Alex Fang and Georgios Smyrnis et.al.},
booktitle={Advances in Neural Information Processing Systems},
year={2024},
}

@inproceedings{dohmatob2025less,
  title={Why Less is More (Sometimes): A Theory of Data Curation},
  author={Elvis Dohmatob and Mohammad Pezeshki and Reyhane Askari-Hemmat},
  booktitle={International Conference on Learning Representations},
  year={2026},
}

@inproceedings{wang2023dynaMS,
  title={DynaMS: Dynamic Margin Selection for Efficient Deep Learning },
  author={Jiaxing Wang and Yong Li and Jingwei Zhuo and Xupeng Shi and Weizhong Zhang and Lixing Gong and Tong Tao and Pengzhang Liu and Yongjun Bao and Weipeng Yan},
  booktitle={International Conference on Learning Representations},
  year={2023},
}

@article{mizrahi2025target,
  author       = {David Mizrahi and Anders Boesen Lindbo Larsen and Jesse Allardice and Suzie Petryk and Yuri Gorokhov and Jeffrey Li and Alex Fang and Josh Gardner and Tom Gunter and Afshin Dehghan},
  title        = {Language Models Improve When Pretraining Data Matches Target Tasks},
  journal      = {arXiv preprint},
  volume       = {arXiv:2507.12466},
  year         = {2025}
}
\bibliographystyle{iclr2027_conference}

\appendix
\newpage
\section{Appendix}
\label{sec:appendix}

\subsection{Two-Teacher Mixture with Heterogeneous Difficulty}
\label{app:mixture}

We instantiate data mixture in a two-domain teacher--student model. A mixture $\bw$ determines how the total of $P$ training examples is allocated across the two domains, with $P_k=w_kP$. All examples are then used jointly to fit one regularized logistic-regression student. We characterize its learning curves using the multi-environment state-evolution framework of \citet{bruno2021learningcurve}.

\paragraph{Heterogeneous-difficulty setup.} We consider two domains with orthogonal teachers $\bB_1,\bB_2\in\mathbb R^N$ satisfying $\|\bB_1\|=\|\bB_2\|=\sqrt N$ and $\bB_1^\top\bB_2=0$. The training set contains $P=\alpha N$ examples. A mixture policy $\bw=(w_1,w_2)\in\mathcal W:= \{\bw\in\mathbb R^K\mid\sum_{k=1}^K w_k=1,\ w_k>0\}$ assigns $P_k=w_kP$ examples to domain $k$, giving the domain-specific sample ratio $\alpha_k=P_k/N=w_k\alpha$; here $K=2$. Examples from domain $k$ are generated according to $\bx_i^k\sim\mathcal N(\mathbf0,\mathbf\Sigma_k)$ and $y_i^k=\operatorname{sign}(\bB_k^\top\bx_i^k)$. Both domains share a student $\bJ\in\mathbb R^N$ with normalized field $h(\bJ,\bx)=\bJ^\top\bx/\sqrt N$. Define the logistic margin loss $\ell_{\log}(r)=\log(1+e^{-r})$. The trained student is
\begin{equation}
\bJ^\star(\bw) = \arg\min_{\bJ} \left\{ \sum_{k=1}^{2}\sum_{i=1}^{P_k} \ell_{\log}\!\left( y_i^k\frac{\bJ^\top\bx_i^k}{\sqrt N} \right) +
\frac{\lambda_{\mathrm{reg}}}{2}\|\bJ\|^2 \right\},
\label{eq:mixture-logistic}
\end{equation}
where $\lambda_{\mathrm{reg}}>0$ is the regularization coefficient. Domain 1 is isotropic, with $\mathbf\Sigma_1=\mathbf I$. We make domain 2 harder by introducing nuisance variance orthogonal to its teacher:
\begin{equation}
\mathbf\Sigma_2
=
\frac{\bB_2\bB_2^\top}{N}
+
(1+\nu^2)
\left(
\mathbf I-\frac{\bB_2\bB_2^\top}{N}
\right),
\label{eq:sigma2}
\end{equation}
where $\nu\geq0$ controls the nuisance level. Since
$\bB_2\bB_2^\top/N$ projects onto the teacher direction,
$\mathbf\Sigma_2$ has eigenvalue $1$ along $\bB_2$ and eigenvalue
$1+\nu^2$ in every orthogonal direction. The nuisance therefore leaves the
labeling rule and Bayes boundary unchanged while increasing the finite-sample
difficulty of estimating $\bB_2$.

\paragraph{State-evolution structure.}
We summarize the trained student using a finite set of order parameters. The
overlap $m_k=\bJ^\top\bB_k/N$ measures the alignment of the student with the
teacher of domain $k$, while $q=\|\bJ\|^2/N$ is its normalized squared norm.
We further introduce the response parameter $V_k$, which measures how
sensitively the fitted student changes under a small perturbation to examples
from domain $k$.
To enforce the definitions of these order parameters in the saddle-point
derivation, we introduce the associated dual variables
$\widehat m_k$, $\widehat V_k$, and $\widehat q_k$, analogous to Lagrange
multipliers. These are auxiliary quantities rather than additional model
parameters. The resulting self-consistent equations consist of two coupled
parts. The loss-side equations determine
$(\widehat m_k,\widehat V_k,\widehat q_k)$ from $(m_k,V_k,q)$ through the
domain-specific logistic loss, while the model-side equations map the dual
variables back to $(m_k,V_k,q)$ through the domain covariances. A joint fixed
point of these equations determines the asymptotic properties of the trained
student.

\paragraph{Loss-side equations.}
We specialize the general state-evolution equations of
\citet{bruno2021learningcurve} to the logistic loss and the two-domain
covariance structure defined above. For domain $k$, let $Q_k$ denote the
variance of the effective student field. The covariance structure gives
$Q_1=q$ and
\begin{equation}
Q_2
=
q+\nu^2(q-m_2^2).
\label{eq:qeff}
\end{equation}
The additional term in $Q_2$ arises from the component of $\bJ$ orthogonal to
$\bB_2$, whose variance is amplified by the nuisance level $\nu$.
Let $V_k^\star=1-m_k^2/Q_k$ denote the conditional variance of the teacher
field given the student field in domain $k$. Specializing the loss-side
saddle-point conditions yields
\begin{align}
\widehat m_k
&=
\frac{\alpha_k}{V_k}
\mathcal I_m(m_k,Q_k,V_k,V_k^\star),
\label{eq:hat-m}
\\
\widehat V_k
&=
\alpha_k
\left[
\frac{1}{V_k}
-
\frac{\mathcal I_V(m_k,Q_k,V_k,V_k^\star)}{V_k^2}
\right],
\label{eq:hat-V}
\\
\widehat q_k
&=
\frac{\alpha_k}{V_k^2}
\mathcal I_Q(m_k,Q_k,V_k,V_k^\star),
\label{eq:hat-q}
\end{align}
where $\mathcal I_m$, $\mathcal I_V$, and $\mathcal I_Q$ are scalar Gaussian
expectations induced by the logistic loss and its proximal operator. Their
explicit expressions are provided below.
The domain difficulty first enters these equations through $Q_2$. By changing
the effective field distribution of domain $2$, the nuisance variance modifies
its conjugate parameters and, consequently, its finite-sample learning
dynamics.

\paragraph{Model-side update and spectral decomposition.}
The conjugate responses define the effective precision operator
\begin{equation}
\mathbf D
=
\lambda_{\mathrm{reg}}\mathbf I
+
\widehat V_1\mathbf\Sigma_1
+
\widehat V_2\mathbf\Sigma_2.
\label{eq:effective-precision}
\end{equation}
The model-side update requires its inverse $\mathbf D^{-1}$. We therefore
decompose the parameter space into
$\operatorname{span}(\bB_1)$,
$\operatorname{span}(\bB_2)$, and
$\mathcal B=\operatorname{span}(\bB_1,\bB_2)^\perp$.
The $(N-2)$-dimensional subspace $\mathcal B$ contains directions orthogonal
to both teachers and is referred to as the \emph{bulk}.

The covariance eigenvalues in these subspaces are
\begin{equation}
\begin{array}{c|ccc}
&
\operatorname{span}(\bB_1)
&
\operatorname{span}(\bB_2)
&
\mathcal B
\\
\hline
\mathbf\Sigma_1 & 1 & 1 & 1 \\
\mathbf\Sigma_2 & 1+\nu^2 & 1 & 1+\nu^2
\end{array}.
\label{eq:mixture-spectrum}
\end{equation}
Consequently, the eigenvalues of $\mathbf D$ are
\begin{equation}
d_{\bB_1}
=
d_{\mathrm{bulk}}
=
\lambda_{\mathrm{reg}}+\widehat V_1+(1+\nu^2)\widehat V_2,
\qquad
d_{\bB_2}
=
\lambda_{\mathrm{reg}}+\widehat V_1+\widehat V_2.
\label{eq:denoms}
\end{equation}
The corresponding eigenvalues of $\mathbf D^{-1}$ are their reciprocals.

Using this decomposition, the model-side updates become
\begin{align}
V_1
&=
\frac{1}{d_{\mathrm{bulk}}},
\qquad
V_2
=
\frac{1+\nu^2}{d_{\mathrm{bulk}}},
\label{eq:Vk}
\\
m_1
&=
\frac{\widehat m_1}{d_{\bB_1}},
\qquad
m_2
=
\frac{\widehat m_2}{d_{\bB_2}},
\label{eq:mk}
\\
q
&=
\frac{\widehat m_1^2}{d_{\bB_1}^2}
+
\frac{\widehat m_2^2}{d_{\bB_2}^2}
+
\frac{\widehat q_1+(1+\nu^2)\widehat q_2}
     {d_{\mathrm{bulk}}^2}.
\label{eq:q-me}
\end{align}
Together, the loss-side and model-side updates form a closed
finite-dimensional system for any mixture policy $\bw$ and sample ratio
$\alpha$.

\paragraph{Target readout.}
For a target example from domain $k$, the normalized student field can be
written as
\begin{equation}
h_k
=
m_k u+s_k\sqrt{q-m_k^2}\,\eta,
\qquad
u,\eta\sim\mathcal N(0,1),
\label{eq:readout-me}
\end{equation}
where $s_1=1$ and $s_2=\sqrt{1+\nu^2}$. The first term is aligned with the
teacher, while the second represents the orthogonal component of the student.
The corresponding signal-to-noise ratio is
$\mathrm{SNR}_k=m_k/[s_k\sqrt{q-m_k^2}]$.

Let $\widehat{\bw}=(\widehat w_1,\widehat w_2)\in\mathcal W$ denote the domain
proportions of the target. Its target loss is
\begin{equation}
\mathcal L_{\mathrm{tar}}(\bw;\widehat{\bw})
=
\sum_{k=1}^{2}
\widehat w_k\,
\mathbb E_{u,\eta}
\left[
\ell_{\log}\!\left(\operatorname{sign}(u)h_k\right)
\right].
\label{eq:mixture-target}
\end{equation}
Thus, solving the closed state-evolution system determines the target loss
of any training mixture $\bw$ against any target mixture
$\widehat{\bw}$. Dependence on $\alpha$, $\nu$, and $\lambda_{\mathrm{reg}}$
is suppressed for readability.

\paragraph{Large-budget limit: an analytic target-matching argument.}
Panel~(c) of \Cref{fig:mixture_ratio_evolve} is evaluated at a large but finite
budget. We now identify its limiting optimum without extrapolating the
numerical state evolution. Divide \eqref{eq:mixture-logistic} by
$P=\alpha N$. If $\|\bJ\|^2=O(N)$ and $\lambda_{\mathrm{reg}}$ is fixed, the
normalized ridge contribution is $O(\lambda_{\mathrm{reg}}/\alpha)$ and hence
vanishes as $\alpha\to\infty$. Under convergence of the empirical
domain risks, the limiting student therefore minimizes
\begin{equation}
\begin{aligned}
\bJ_\infty^\star(\bw)
&\in\arg\min_{\bJ}\ \mathcal R_{\bw}(\bJ),
\\
\mathcal R_{\bw}(\bJ)
&=\sum_{k=1}^{2} w_k\,\mathcal L^{\mathrm{pop},k}(\bJ),
\\
\mathcal L^{\mathrm{pop},k}(\bJ)
&=\mathbb E_{\bx\sim\mathcal N(\mathbf0,\mathbf\Sigma_k)}
\!\left[\ell_{\log}\!\left(
\operatorname{sign}(\bB_k^\top\bx)h(\bJ,\bx)
\right)\right].
\end{aligned}
\label{eq:mixture-pop-min}
\end{equation}
The target population risk and the endpoint induced by a training mixture are
\begin{equation}
\mathcal R_{\widehat{\bw}}(\bJ)
=\sum_{k=1}^{2}\widehat w_k\mathcal L^{\mathrm{pop},k}(\bJ),
\qquad
\mathcal E_\infty(\bw;\widehat{\bw})
=\mathcal R_{\widehat{\bw}}\!\left(\bJ_\infty^\star(\bw)\right).
\label{eq:mixture-target-pop-risk}
\end{equation}
Choosing $\bw=\widehat{\bw}$ makes the limiting training objective exactly the
target objective. Consequently,
\begin{equation}
\mathcal E_\infty(\widehat{\bw};\widehat{\bw})
\leq
\mathcal E_\infty(\bw;\widehat{\bw}),
\qquad \forall\,\bw\in\mathcal W,
\label{eq:mixture-target-matching-optimal}
\end{equation}
so target matching is a global optimum. This argument is analytic and does not
depend on the nuisance level $\nu$. The \emph{location in mixture space} of the
matching optimum is therefore independent of difficulty, although the
population minimizer and attained target loss generally still depend on $\nu$
through $\Sigma_2$.

If the target risk has a unique minimizer and the baseline induces a different
population minimizer, the inequality in
\eqref{eq:mixture-target-matching-optimal} is strict. These conditions hold in
the nondegenerate two-teacher setting considered here: the Gaussian
covariances are full rank, the mixed logistic population risk is strictly
convex at its finite optimum, and distinct teachers do not share a stationary
point of their individual risks. Two regimes follow:
\begin{itemize}[leftmargin=*,itemsep=1pt]
\item \emph{No gap} ($\widehat{\bw}=\bw_{\text{baseline}}$): matching returns the
baseline, so $\mathcal E_\infty(\bw;\widehat{\bw})\ge
\mathcal E_\infty(\bw_{\text{baseline}};\widehat{\bw})$
for all $\bw$ and no reallocation lowers the endpoint---acceleration only.
\item \emph{With a correctable composition gap}
($\widehat{\bw}\neq\bw_{\text{baseline}}$): the baseline minimizes a differently
weighted population risk. Under the nondegeneracy conditions above, the two
weightings induce different minimizers, and hence
$\mathcal E_\infty(\widehat{\bw};\widehat{\bw})<
\mathcal E_\infty(\bw_{\text{baseline}};\widehat{\bw})$---advancement.
\end{itemize}
Thus the numerical curve at $\alpha=100$ in panel~(c) approximates an exactly
target-matched population optimum; it is not the sole evidence for that
endpoint.

\paragraph{Consequences for acceleration and advancement.}
At finite sample ratios, the unequal field variances in \Cref{eq:qeff} cause
the two domains to be learned at different rates. Allocating more data to the
harder domain can therefore reduce target loss at small
$\alpha$, as in \Cref{fig:mixture_ratio_evolve}(b). When the
baseline already matches the target, this advantage vanishes with increasing
budget and is acceleration. When the target composition differs from the
baseline, \eqref{eq:mixture-target-matching-optimal} instead gives a strictly
better population endpoint under the stated nondegeneracy conditions, which is
the advancement effect in panel~(c).

\paragraph{Logistic integrals.}
For completeness, we specify the Gaussian expectations in
\Cref{eq:hat-m,eq:hat-V,eq:hat-q}. Let
$\omega=\sqrt Q\,\xi$ with $\xi\sim\mathcal N(0,1)$,
$\omega_\star=(m/\sqrt Q)\xi$, and
\begin{equation}
a(\xi)
=
\operatorname{erf}\!\left(
\frac{(m/\sqrt Q)\xi}{\sqrt{2V^\star}}
\right).
\label{eq:teacher-weight}
\end{equation}
Define
\begin{equation}
\eta^\pm
=
\operatorname{prox}_{V\ell_{\log}(\pm\cdot)}(\omega)
=
\arg\min_r
\left\{
\frac{(r-\omega)^2}{2V}
+
\ell_{\log}(\pm r)
\right\}.
\label{eq:logistic-prox}
\end{equation}
The required integrals are
\begin{align}
\mathcal I_m
&=
\frac{1}{\sqrt{2\pi V^\star}}
\mathbb E_\xi
\left[
e^{-\frac{\omega_\star^2}{2V^\star}}
(\eta^+-\eta^-)
\right],
\label{eq:Im}
\\
\mathcal I_V
&=
\frac{1}{2}\mathbb E_\xi
\left[
\frac{1+a(\xi)}
     {1/V+\frac14\operatorname{sech}^2(\eta^+/2)}
+
\frac{1-a(\xi)}
     {1/V+\frac14\operatorname{sech}^2(\eta^-/2)}
\right],
\label{eq:IV}
\\
\mathcal I_Q
&=
\frac{1}{2}\mathbb E_\xi
\left[
(1+a(\xi))(\eta^+-\omega)^2
+
(1-a(\xi))(\eta^--\omega)^2
\right].
\label{eq:IQ}
\end{align}

\paragraph{Reduction checks and numerical solution.}
When $\nu=0$, the two domains have identical covariance,
$d_{\bB_1}=d_{\bB_2}=d_{\mathrm{bulk}}
=\lambda_{\mathrm{reg}}+\widehat V_1+\widehat V_2$, and $V_1=V_2$.
The system then reduces to the isotropic two-teacher logistic state evolution.
We solve the fixed-point equations using a nonuniform composite quadrature for
the Gaussian expectations. At large sample ratios, $V^\star$ approaches zero and
$a(\xi)$ becomes sharply concentrated, causing fixed Gauss--Hermite grids to
under-resolve the integrands. Refining the grid around the shrinking transition
layer avoids this numerical artifact. The isotropic reduction at $\nu=0$ and agreement with
finite-dimensional simulations provide independent checks of the solution.

\subsection{Selection under Matched and Shifted Targets}
\label{app:imperfect}

Selection starts from $P=\alpha N$ candidate examples and retains a fraction
$f$, leaving a training budget $f\alpha N$. We first consider a matched target
and then extend the model to a shifted target, where label consistency provides
an additional selection signal. We write the target-error learning curve as
$\mathcal E_\alpha(f;\phi)$, where $\phi$ is the angle between the training and
target teachers; dependence on the fixed $\lambda_{\mathrm{sel}}$ in the
shifted-target policy is suppressed.

\subsubsection{Matched target: classical margin selection}
\label{app:single}
Let $\bB\in\mathbb R^N$, $\|\bB\|=\sqrt N$, be both the training and target
teacher. The $P=\alpha N$ candidates are generated by
$\bx_i\sim\mathcal N(\mathbf0,\mathbf I)$ and
$y_i=\operatorname{sign}(\bB^\top\bx_i)$. Writing the teacher margin as
$z=\bB^\top\bx/\sqrt N$ (and $z_i=\bB^\top\bx_i/\sqrt N$ for candidate
$i$), margin selection retains $|z|<\gamma_f$, where
\begin{equation}
f=\int_{|z|<\gamma_f}\varphi(z)\,\d z
\end{equation}
is the retained fraction and $\varphi$ is the standard Gaussian density. The
selected set is $S_f=\{i:|z_i|<\gamma_f\}$. We use the spherical max-margin
student
\begin{equation}
\bJ^\star
\in
\arg\max_{\|\bJ\|=\sqrt N}
\min_{i\in S_f}
\frac{y_i\bJ^\top\bx_i}{\sqrt N},
\qquad
\kappa
=
\min_{i\in S_f}
\frac{y_i(\bJ^\star)^\top\bx_i}{\sqrt N}.
\label{eq:app-single-student-compact}
\end{equation}
This has the same direction as the usual minimum-norm hard-margin solution;
the spherical form makes the normalization used by the replica equations
explicit.

Rather than repeating the replica calculation, we use the arbitrary-margin
single-teacher equations of \citet[][Eqs.~66--67]{sorscher2022beyondsl}; for a
truncated Gaussian margin distribution, they reduce to their Eqs.~68--69. The
retained examples have sample ratio $f\alpha$ and conditional margin density
\begin{equation}
p_f(z)
=\frac{\varphi(z)\mathbf 1\{|z|<\gamma_f\}}{f}.
\label{eq:selected-margin-density}
\end{equation}
In the high-dimensional limit, the empirical selection threshold and retained
margin distribution concentrate at $\gamma_f$ and $p_f$, respectively.
The cited equations depend on the product of the sample ratio and margin
density. Substituting $(f\alpha,p_f)$ therefore gives
$f\alpha p_f(z)=\alpha\varphi(z)\mathbf 1\{|z|<\gamma_f\}$, exactly the
exposure induced by selecting from the candidate set. Thus the normalized
student--teacher overlap $R=(\bJ^\star)^\top\bB/N$ and maximum stability
$\kappa$ satisfy
\begin{align}
R
&=
2f\alpha\int_{-\infty}^{\kappa}\!\d t
\int_0^\infty\!\d z\,
p_f(z)
\frac{e^{-\frac{(t-Rz)^2}{2(1-R^2)}}}
{\sqrt{2\pi(1-R^2)}}
 (\kappa-t)
\frac{z-Rt}{1-R^2},
\label{eq:app-single-R}
\\[2pt]
1-R^2
&=
2f\alpha\int_{-\infty}^{\kappa}\!\d t
\int_0^\infty\!\d z\,
p_f(z)
\frac{e^{-\frac{(t-Rz)^2}{2(1-R^2)}}}
{\sqrt{2\pi(1-R^2)}}
(\kappa-t)^2.
\label{eq:app-single-norm}
\end{align}
The factor of $2$ uses the symmetry between positive and negative teacher
margins, allowing the $z$-integral to be restricted to $z>0$.
The target error is $\mathcal E_\alpha(f;0)=\arccos(R)/\pi$. We solve
\Cref{eq:app-single-R,eq:app-single-norm} by adaptive quadrature; $f=1$
recovers the standard unpruned perceptron equations. These curves give the
matched-target selection result in \Cref{fig:mixture_ratio_evolve}(e).

\subsubsection{Shifted target: consistency-aware selection}
We now let the target teacher differ from the training teacher. Applying a
target-margin rule alone is insufficient: labels still come from the training
teacher, so examples near the target boundary may carry labels that conflict
with the target. We therefore model the additional signal used by target-aware
selectors through a consistency filter. This remains selection rather than
label correction: inconsistent examples are removed, and no label is changed.

\paragraph{Setup and selection rule.}
Let the training and target teachers be $\bB$ and $\bB_{\mathrm{tar}}$, with
$\|\bB\|=\|\bB_{\mathrm{tar}}\|=\sqrt N$ and
\begin{equation}
R_0=\frac{\bB^\top\bB_{\mathrm{tar}}}{N}=\cos\phi.
\end{equation}
We take $0<\phi<\pi/2$ in this shifted-target analysis; the matched endpoint
$\phi=0$ is handled separately above.
For an example $\bx$, define
\begin{equation}
u=\frac{\bB^\top\bx}{\sqrt N},\qquad
z=\frac{\bB_{\mathrm{tar}}^\top\bx}{\sqrt N},\qquad
y=\operatorname{sign}(u),
\end{equation}
so $(u,z)$ is jointly standard Gaussian with correlation $R_0$. Let
$A=\{\operatorname{sign}(u)=\operatorname{sign}(z)\}$ denote label
consistency. We rank candidates by the single target-aware score
\begin{equation}
q_{\lambda_{\mathrm{sel}}}(u,z)
=|z|+\lambda_{\mathrm{sel}}\mathbf 1_{A^c},
\qquad \lambda_{\mathrm{sel}}>0.
\label{eq:consistency-selection-rule}
\end{equation}
When $R_0=1$, $A$ holds almost surely and this score reduces to the
matched-target margin rule above.
We retain the lowest-scoring fraction $f$. Equivalently, for a threshold
$\tau_f$, a consistent example is retained when $|z|<\tau_f$, whereas an
inconsistent one is retained only when
$|z|<(\tau_f-\lambda_{\mathrm{sel}})_+$, where $(r)_+=\max\{r,0\}$. The
threshold is determined by
\begin{equation}
f
=2\int_0^{\tau_f}\varphi(z)
\Phi\!\left(\frac{R_0z}{\sqrt{1-R_0^2}}\right)\d z
+2\int_0^{(\tau_f-\lambda_{\mathrm{sel}})_+}\varphi(z)
\left[1-\Phi\!\left(\frac{R_0z}{\sqrt{1-R_0^2}}\right)\right]\d z,
\label{eq:consistency-keep-fraction}
\end{equation}
where $\varphi$ and $\Phi$ are the standard Gaussian density and cdf. Unlike a
lexicographic ``consistency first, margin second'' rule, this score changes the
selected distribution continuously as $f$ varies. The two terms nevertheless
retain distinct roles: consistency makes the gap correctable, whereas margin
controls finite-budget informativeness. We write
$s_f(u,z)=\mathbf 1\{q_{\lambda_{\mathrm{sel}}}(u,z)<\tau_f\}$ for the resulting
selection indicator. Panel~(f) uses $\lambda_{\mathrm{sel}}=0.75$.

Applied to the $P=\alpha N$ candidates, $s_f$ retains $fP$ examples and hence
uses sample ratio $f\alpha$, as in the matched-target case.

\paragraph{Conditional field and truncated moments.}
The selected student is again the spherical max-margin solution, with
$\|\bJ\|=\sqrt N$; let $\kappa$ denote its maximum stability. Hence
$R=\bJ^\top\bB/N$ and
$\rho=\bJ^\top\bB_{\mathrm{tar}}/N$ are normalized cosine overlaps. For
$t=\bJ^\top\bx/\sqrt N$, Gaussian conditioning gives
\begin{equation}
t\mid(u,z)\sim\mathcal N(au+bz,\Sigma^2),
\end{equation}
where
\begin{equation}
a=\frac{R-\rho R_0}{1-R_0^2},\qquad
b=\frac{\rho-RR_0}{1-R_0^2},\qquad
\Sigma^2=1-\frac{R^2+\rho^2-2R\rho R_0}{1-R_0^2}.
\label{eq:consistency-conditional-field}
\end{equation}
Let $h=yt$, $\mu=au+bz$, and
\begin{equation}
\eta=\frac{\kappa-y\mu}{\Sigma},\quad
\mathcal I_0=\Phi(\eta),\quad
\mathcal I_1=\Sigma[\eta\Phi(\eta)+\varphi(\eta)],
\end{equation}
\begin{equation}
\mathcal I_2=\Sigma^2[(1+\eta^2)\Phi(\eta)
+\eta\varphi(\eta)].
\end{equation}
These are the truncated Gaussian moments generated by the active hard-margin
constraints. For any $F(u,z)$, denote expectation under the normalized selected
distribution by
\begin{equation}
\langle F\rangle_f
=\frac{\mathbb E[F(u,z)s_f(u,z)]}{f}.
\end{equation}

\paragraph{Finite-budget saddle-point equations.}
The replica-symmetric, max-margin limit gives the following closed system for
$(\kappa,R,\rho)$:
\begin{align}
f\alpha\langle\mathcal I_2\rangle_f
&=\Sigma^2,
\label{eq:consistency-q}\\
f\alpha\left[
\frac{\langle y(u-R_0z)\mathcal I_1\rangle_f}{1-R_0^2}
+a\langle\mathcal I_0\rangle_f
\right]&=a,
\label{eq:consistency-R}\\
f\alpha\left[
\frac{\langle y(z-R_0u)\mathcal I_1\rangle_f}{1-R_0^2}
+b\langle\mathcal I_0\rangle_f
\right]&=b.
\label{eq:consistency-rho}
\end{align}
The target error is $\mathcal E_\alpha(f;\phi)=\arccos(\rho)/\pi$. These
equations generate the curves in
\Cref{fig:mixture_ratio_evolve}(f); numerical quadrature and finite-dimensional
hard-margin simulations use the same candidate set and policy.

\paragraph{An interior point on the score path.}
The score in \eqref{eq:consistency-selection-rule} itself contains an explicit
interior policy for every finite $\lambda_{\mathrm{sel}}>0$. Consider the threshold
$\tau_f=\lambda_{\mathrm{sel}}$. A consistent example is then retained when
$|z|<\lambda_{\mathrm{sel}}$, whereas an inconsistent example would have to satisfy
$|z|<0$ and is therefore discarded up to a measure-zero boundary.  The
corresponding retained fraction is
\begin{equation}
f_{\lambda_{\mathrm{sel}}}
=\Pr[A,|z|<\lambda_{\mathrm{sel}}]
=2\int_0^{\lambda_{\mathrm{sel}}}\varphi(z)
\Phi\!\left(\frac{R_0z}{\sqrt{1-R_0^2}}\right)\d z,
\qquad 0<f_{\lambda_{\mathrm{sel}}}<1.
\label{eq:consistency-interior-policy}
\end{equation}
Thus $f_{\lambda_{\mathrm{sel}}}$ is an interior point on the same
$\lambda_{\mathrm{sel}}$-parameterized policy path
used in panel~(f), rather than a certificate supplied by a separate
consistency-only rule. More generally, $f\leq f_{\lambda_{\mathrm{sel}}}$ corresponds to
$\tau_f\leq\lambda_{\mathrm{sel}}$ and contains no inconsistent examples, while
$f>f_{\lambda_{\mathrm{sel}}}$ corresponds to $\tau_f>\lambda_{\mathrm{sel}}$ and retains a positive-measure
set of inconsistent examples near the target boundary.

\paragraph{Large-budget endpoint of the score path.}
At $f=f_{\lambda_{\mathrm{sel}}}$, multiplying each retained input by its
training label $y$ gives signed fields $yu=|u|>0$ and $yz=|z|>0$. In the
$(\bB,\bB_{\mathrm{tar}})$ plane, the local support is therefore a wedge
bounded by the two teacher decision boundaries. The truncation
$|z|<\lambda_{\mathrm{sel}}$ changes its radial extent but not its tangent cone
at the intersection of the boundaries. As the sample ratio grows, the active
constraints approach both sides of this wedge. Within this isotropic model,
the direction maximizing the angular margin to the two sides is their
bisector; an orthogonal component cannot improve the margin because it carries
no systematic label signal. This support-geometry characterization gives
\begin{equation}
\frac{\bJ_\infty}{\|\bJ_\infty\|}
=\frac{\bB+\bB_{\mathrm{tar}}}{\|\bB+\bB_{\mathrm{tar}}\|},
\qquad
\rho_\infty=\cos\frac{\phi}{2}.
\label{eq:consistency-bisector}
\end{equation}
By contrast, full retention preserves all labels from $\bB$, so its
large-sample solution aligns with $\bB$ and has target overlap $\cos\phi$.
Consequently,
\begin{equation}
\mathcal E_\infty(f_{\lambda_{\mathrm{sel}}};\phi)
=\frac{\phi}{2\pi}
<\frac{\phi}{\pi}
=\mathcal E_\infty(1;\phi).
\label{eq:consistency-advancement}
\end{equation}
Thus the family contains a fixed interior policy that strictly beats full
retention at the endpoint. The particular value $\lambda_{\mathrm{sel}}=0.75$ is not needed
for existence: every $\lambda_{\mathrm{sel}}>0$ defines a nontrivial
$f_{\lambda_{\mathrm{sel}}}$ through
\eqref{eq:consistency-interior-policy}.  Consistency supplies the endpoint
shift, whereas the margin term locates that shift on the continuous selection
path and controls which consistent constraints are retained.

The same argument clarifies the two sides of the threshold. For any fixed
$0<f\leq f_{\lambda_{\mathrm{sel}}}$, the selected support has the same
agreement wedge near the origin and hence the same bisector endpoint. Once
$f>f_{\lambda_{\mathrm{sel}}}$, the cutoff
$|z|<(\tau_f-\lambda_{\mathrm{sel}})$ admits inconsistent examples in an open
neighborhood of the intersection. Any separating direction with a component
orthogonal to $\bB$ then violates some training-label constraint in the
population support; the asymptotic direction therefore returns to $\bB$.
Thus, within this idealized model, the support-level prediction is
\begin{equation}
\mathcal E_\infty(f;\phi)
=
\begin{cases}
\phi/(2\pi), & 0<f\leq f_{\lambda_{\mathrm{sel}}},\\
\phi/\pi, & f_{\lambda_{\mathrm{sel}}}<f\leq1.
\end{cases}
\label{eq:consistency-endpoint-path}
\end{equation}
Finite candidate sets smooth this sharp support-level transition.

\paragraph{Why the finite-budget curve has an interior optimum.}
The finite-budget tradeoff is now explicit. At
$f=1$ the learner uses the full candidate set and approaches the training teacher.
Near $f=f_{\lambda_{\mathrm{sel}}}<1$, the score has removed the inconsistent constraints and
the learner approaches the improved bisector endpoint in
\eqref{eq:consistency-advancement}. If $f$ is reduced too far at fixed
$\alpha$, however, the retained training ratio $f\alpha$ vanishes and target
accuracy returns to chance. Continuity of the saddle equations then implies
that, for all sufficiently large but finite $\alpha$, a maximum is attained
at some
\begin{equation}
f^\star(\alpha)\in(0,1)
\end{equation}
and exceeds the value at $f=1$.  The optimum lies near the point where the
last inconsistent constraints disappear; for the parameters of panel~(f),
this is $f_{\lambda_{\mathrm{sel}}}$. As $\alpha\to\infty$, the small-$f$
sample-efficiency penalty disappears for every fixed $f>0$, so the smooth
finite-budget peak need not converge to a unique optimizer.  Instead it can
broaden into the improved plateau in
\eqref{eq:consistency-endpoint-path}.  What persists is the advancement: the
fixed interior policy $f=f_{\lambda_{\mathrm{sel}}}$ remains strictly better than full
retention in the large-budget limit.


\subsection{Formalizing Observation~\ref{prop:impossibility_main}:
Bi-level Optimality under No Gap}
\label{app:impossibility}

We now formalize Observation~\ref{prop:impossibility_main} for mixture and selection. For mixture, the reference policy is the target-matching allocation $\widehat{\bw}$. When $\widehat{\bw}=\bw_{\mathrm{baseline}}$, the resulting optimality statement is exactly the no-gap conclusion of Observation~\ref{prop:impossibility_main}. For selection, the corresponding baseline is the uniform soft allocation $\bar w_i=f$, whose operational counterpart is uniform random hard selection. The connection holds in the large-sample limit.

\paragraph{Domain-level mixture.}
We begin with data mixture, where the allocation variable assigns training exposure across domains. Throughout this section the losses are \emph{population} quantities, written $\mathcal L^{\mathrm{pop}}$ to distinguish them from the empirical, seen-data losses $\mathcal L$ measured by the $\Delta$-diagnostic (\Cref{sec:delta_test}). 

Let $x$ denote a data example (including its label, when applicable), and let
$\ell(\btheta;x)$ be the loss of model $\btheta$ on that example. For domain
$k$, define
\begin{equation*}
\mathcal L_{\mathrm{train}}^{\mathrm{pop},k}(\btheta)
:=
\mathbb E_{x\sim\mathcal P_{\mathrm{train}}^k}[\ell(\btheta;x)],
\qquad
\mathcal L_{\mathrm{tar}}^{\mathrm{pop},k}(\btheta)
:=
\mathbb E_{x\sim\mathcal P_{\mathrm{target}}^k}[\ell(\btheta;x)].
\end{equation*}

A mixture policy $\bw$ can be realized either by sampling $P_k=w_kP$ unit-weight examples from each domain, or by using $w_k$ directly as weighting coefficients on the domain losses. For an additive per-example loss, the empirical risk under the sampling-based implementation satisfies
\begin{equation}
\begin{aligned}
\frac{1}{P}
\sum_{k=1}^K
\sum_{i=1}^{P_k}
\ell(\btheta;x_i^k)
&=
\sum_{k=1}^K
\frac{P_k}{P}
\left[
\frac{1}{P_k}
\sum_{i=1}^{P_k}
\ell(\btheta;x_i^k)
\right]
\\
&\xrightarrow[P\to\infty]{}
\sum_{k=1}^K
w_k\,
\mathbb E_{x\sim\mathcal P_{\mathrm{train}}^k}
[\ell(\btheta;x)]
\\
&=
\sum_{k=1}^K
w_k\mathcal L_{\mathrm{train}}^{\mathrm{pop},k}(\btheta).
\end{aligned}
\label{eq:mixture-implementation-equivalence}
\end{equation}

Assuming i.i.d.\ sampling within each domain and uniform convergence of the empirical domain risks to their population counterparts, the sampling-based objective converges uniformly to the loss-weighted population objective. Thus, the two implementations induce the same limiting population problem. At finite $P$, they may differ because $w_kP$ must be rounded and the empirical domain averages fluctuate, but these differences vanish as $P_k\to\infty$ for every $k$. This domain-level equivalence is distinct from the example-level soft relaxation introduced below. We therefore state the bi-level mixture result directly in the loss-weighted population form.

\begin{observation}[Target-matching optimality for mixture]
\label{prop:impossibility_mixture}
Consider the bi-level mixture problem
\begin{equation}
\bw^\star
\in
\arg\min_{\bw\in\mathcal W}
\sum_{k=1}^K
\widehat w_k
\mathcal L_{\mathrm{tar}}^{\mathrm{pop},k}
\bigl(\btheta^\star(\bw)\bigr),
\quad\text{s.t.}\quad
\btheta^\star(\bw)
\in
\arg\min_{\btheta}
\sum_{k=1}^K
w_k\mathcal L_{\mathrm{train}}^{\mathrm{pop},k}(\btheta),
\label{eq:bilevel}
\end{equation}
where $\widehat{\bw}\in\mathcal W$ is the target domain composition. Suppose
$\mathcal P_{\mathrm{train}}^k=\mathcal P_{\mathrm{target}}^k$ and hence
$\mathcal L_{\mathrm{train}}^{\mathrm{pop},k}=\mathcal L_{\mathrm{tar}}^{\mathrm{pop},k}
=:\mathcal L^{\mathrm{pop},k}$
for every $k$. Then the target-matching allocation
$\bw=\widehat{\bw}$ is a global minimizer of \eqref{eq:bilevel}. In
particular, if the target weights are uniform, then
$\bw=\tfrac{1}{K}\mathbf1$ is globally optimal.
\end{observation}

\begin{proof}
Under the no-gap condition, the training and target population losses coincide within every domain. When $\bw=\widehat{\bw}$, the inner problem becomes
\begin{equation}
\btheta^\star(\widehat{\bw}) \in \arg\min_{\btheta} \sum_{k=1}^K \widehat w_k \mathcal L^{\mathrm{pop},k}(\btheta), 
\end{equation}
which is exactly the outer target objective viewed as a function of $\btheta$. Therefore, its value at $\btheta^\star(\widehat{\bw})$ cannot exceed its value at the model $\btheta^\star(\bw)$ induced by any other mixture $\bw\in\mathcal W$:
\begin{equation} 
\sum_{k=1}^K\widehat w_k \mathcal L^{\mathrm{pop},k} \bigl(\btheta^\star(\widehat{\bw})\bigr) \leq \sum_{k=1}^K\widehat w_k \mathcal L^{\mathrm{pop},k} \bigl(\btheta^\star(\bw)\bigr).
\label{eq:matching-inner-optimality}
\end{equation}
Thus, $\widehat{\bw}$ is a global minimizer of the outer objective.
\end{proof}
Setting $\widehat{\bw}=\bw_{\mathrm{baseline}}$ thus recovers the mixture instantiation of Observation~\ref{prop:impossibility_main}.

The assumption $\mathcal P_{\mathrm{train}}^k=\mathcal P_{\mathrm{target}}^k$ is a distribution-level condition. It implies equality of the corresponding population losses but does not assert equality between individual training and target examples. This distinction becomes important for data selection, where exposure is allocated at the level of individual sample rather than domains.

\paragraph{Example-level selection.}
We next express selection as a per-example bi-level problem. Let
$\mathcal D_{\mathrm{train}}=\{x_i\}_{i=1}^{P}$ be the candidate
pool. A hard selection policy is represented exactly by binary indicators in
\begin{equation}
\mathcal W_f^{\mathrm{hard}}
=
\left\{\bw\in\{0,1\}^P:
\sum_{i=1}^{P}w_i=fP\right\}.
\end{equation}
For the population-level argument, we consider the continuous relaxation
\begin{equation}
\mathcal W_f^{\mathrm{soft}}
=
\left\{\bw\in[0,1]^P:
\sum_{i=1}^{P}w_i=fP\right\}.
\end{equation}
The relaxed bi-level problem is:
\begin{equation}
\min_{\bw\in\mathcal W_f^{\mathrm{soft}}}
\mathcal L_{\mathrm{tar}}^{\mathrm{pop}}
\bigl(\btheta^\star(\bw)\bigr)
\quad\text{subject to}\quad
\btheta^\star(\bw)
\in
\arg\min_{\btheta}
\frac{1}{fP}
\sum_{i=1}^{P}
w_i\ell(\btheta;x_i).
\label{eq:selection-bilevel}
\end{equation}
Here, selection reweights individual examples in the finite candidate pool, whereas evaluation is performed on the target distribution as a whole. We therefore define the no-gap condition as $\mathcal P_{\mathrm{train}}=\mathcal P_{\mathrm{target}}$. The inner objective is a weighted loss over the selected examples, while the outer
objective is the target population loss. Under uniform selection, the former
converges to the latter as $P\to\infty$, which motivates the asymptotic result
below.

\begin{observation}[Asymptotic optimality of uniform selection]
\label{prop:selection-impossibility}
Suppose the candidate and target examples are drawn from the same distribution, $\mathcal P_{\mathrm{train}}=\mathcal P_{\mathrm{target}}$. Assume that empirical risk minimization is consistent over the model class (e.g., via uniform convergence). Then the uniform soft allocation $\bar w_i=f$ is a global minimizer of \eqref{eq:selection-bilevel} in the
large-sample limit.
\end{observation}

\begin{proof}
Under the no-gap condition, the population training and target objectives coincide:
\begin{equation}
\mathcal L_{\mathrm{tar}}^{\mathrm{pop}}(\btheta) = \mathbb E_{x\sim\mathcal P_{\mathrm{target}}} [\ell(\btheta;x)] = \mathbb E_{x\sim\mathcal P_{\mathrm{train}}} [\ell(\btheta;x)].
\label{eq:selection-no-gap}
\end{equation}
For the uniform allocation $\bar w_i=f$, the inner objective becomes
\begin{equation}
\frac{1}{fP} \sum_{i=1}^{P} \bar w_i\ell(\btheta;x_i) = \frac{1}{P} \sum_{i=1}^{P} \ell(\btheta;x_i).
\label{eq:uniform-selection-objective}
\end{equation}
As $P\to\infty$, consistency of empirical risk minimization implies that $\btheta^\star(\bar\bw)$ approaches the set of population target minimizers. Hence, in the large-sample limit, for every admissible allocation $\bw\in\mathcal W_f^{\mathrm{soft}}$,
\begin{equation}
\mathcal L_{\mathrm{tar}}^{\mathrm{pop}} \bigl(\btheta^\star(\bar\bw)\bigr) \leq \mathcal L_{\mathrm{tar}}^{\mathrm{pop}} \bigl(\btheta^\star(\bw)\bigr).
\end{equation}
Therefore, no alternative selection allocation achieves a better limiting target endpoint than uniform selection.
\end{proof}

The allocation $\bar w_i=f$ is an analytical reference rather than a literal subset: it reproduces the empirical risk of the full candidate pool while satisfying the same total weight constraint. Its hard-selection analogue is uniform random selection without replacement, under which every example has marginal inclusion probability $f$. Conditional on the candidate pool, the resulting empirical objective is an unbiased estimator of the soft-allocation objective in \eqref{eq:uniform-selection-objective} and concentrates around it as $P\to\infty$. Moreover, $\mathcal W_f^{\mathrm{hard}}\subset\mathcal W_f^{\mathrm{soft}}$, so the result also rules out a lower limiting target loss from any hard selection policy. Thus, uniform random selection is the fixed-retention, asymptotic counterpart of the full-retention reference policy in Observation~\ref{prop:impossibility_main}: both induce the baseline distribution in the population limit. This correspondence does not imply finite-sample equivalence between the soft allocation and any particular random subset.

\paragraph{Connection to Observation~\ref{prop:impossibility_main}.}
The two results instantiate the same no-gap principle at different
granularities. For mixture, target matching is optimal directly at the
population level; setting
$\widehat{\bw}=\bw_{\mathrm{baseline}}$ recovers the mixture case of
Observation~\ref{prop:impossibility_main}. For selection, full retention is the
population reference, and uniform random hard selection approaches the same
endpoint asymptotically. Thus, no gap rules out advancement in both cases,
while still allowing finite-budget acceleration.

\paragraph{Relation to the empirical $\Delta$-diagnostic.}
The observations above are stated in terms of population losses
$\mathcal L^{\mathrm{pop}}$, whereas the $\Delta$-diagnostic of \Cref{sec:delta_test}
is computed from the empirical, seen-data losses $\mathcal L$ of a single
baseline run. The two are connected by the same large-sample limit used in the
proofs: when the baseline visits each example about once in the data-rich
regime, the empirical training loss converges to its population counterpart, so
the no-gap condition $\mathcal P_{\mathrm{train}}=\mathcal P_{\mathrm{target}}$
typically manifests as a vanishing empirical loss gap $\Delta\approx0$, and
\Cref{prop:impossibility_mixture,prop:selection-impossibility} predict that no
allocation can advance the model. This correspondence breaks precisely in the
endogenous regime, where repetition on under-resourced domains drives the
empirical training loss below the population loss: the seen-data $\Delta$ then
becomes nonzero and detects a gap that a purely population statement would miss
(\Cref{app:measuring_delta}). The $\Delta$-diagnostic is thus a useful observable,
finite-sample signature of the population condition these observations
characterize, but it is not equivalent to that condition.

\subsection{Measuring and Interpreting \texorpdfstring{$\Delta$}{Delta}}
\label{app:measuring_delta}

\paragraph{Measurement on seen training data.}
We evaluate $\mathcal L_{\mathrm{train}}$ on the examples used to train the uncurated baseline, rather than on a held-out sample from the training distribution. This distinction matters for an endogenous gap: when a domain is under-resourced, repeated exposure can reduce its loss on the seen examples without producing the same improvement on the target distribution. A held-out training split would not reveal this repetition-induced discrepancy. Because $\Delta$ compares a finite collection of losses rather than distributions directly, it remains an operational diagnostic rather than a complete statistical test: distinct distributions can yield similar losses, and finite-sample fluctuations can produce a small nonzero value.

\paragraph{Relation to curation gain.}
Figure~\ref{fig:replica-gain-vs-delta} evaluates the diagnostic in the two solvable models studied above. In mixture, varying the target composition changes $\Delta$ while the train and target distributions remain matched within each domain. The gain from optimizing the mixture is zero at $\Delta=0$ and increases as the composition gap grows. 
In selection, increasing the teacher angle $\phi$ enlarges the candidate--target mismatch, and the improvement achieved by target-aware selection over full retention increases correspondingly. Thus, despite their different allocation units and performance metrics, both models exhibit the same qualitative relation: a larger gap exposed by the uncurated baseline provides greater headroom for curation.

\paragraph{Null calibration.}
A measured $\Delta$ is never exactly zero, even without a train--target gap, because it depends on the training seed of the baseline and on the finite set of examples. Its scale also depends on the model, tokenizer, and data. We therefore calibrate $\Delta$ against a no-gap null within each experiment family, whose target is held-out data from the training distribution: the uniformly weighted SlimPajama target for mixture and the held-out OpenWebMath target for selection. Let $\mu_{\mathrm{null}}$ and $\sigma_{\mathrm{null}}$ be the mean and standard deviation of $\Delta$ over independently trained baselines in this setting (five for mixture, three for selection). We call $\Delta$ appreciably nonzero when
\begin{equation*}
\zeta=\frac{\Delta-\mu_{\mathrm{null}}}{\sigma_{\mathrm{null}}}>3,
\end{equation*}
the conventional three-standard-deviation criterion. All quantities are computed on existing baseline checkpoints, so no additional training is required. \Cref{tab:delta_null} reports $\Delta$ and $\zeta$ at half the training budget. Every gap setting has $\zeta\geq 140$, far above the threshold of 3.

\begin{table}[h]
\centering
\caption{$\Delta$ at half the training budget and its $\zeta$, the number of null standard deviations above the null mean.}
\label{tab:delta_null}
\small
\begin{tabular}{llcc}
\toprule
Family & Setting & $\Delta$ & $\zeta$ \\
\midrule
Mixture & No gap (null) & $0.076$ & -- \\
 & Endogenous & $0.649$ & $\approx540$ \\
 & Exogenous & $0.468$ & $\approx370$ \\
\midrule
Selection & No gap (null) & $0.016$ & -- \\
 & Shifted target & $0.158$ & $\approx140$ \\
\bottomrule
\end{tabular}
\end{table}

\paragraph{Early evaluation.}
\Cref{tab:delta_mixture_ref} reports $\Delta$ of the Uniform baseline in the three mixture settings at $1/20$, $1/10$, $3/20$, and $1/2$ of the training budget. The exogenous gap is clear and stable from $1/20$ of the budget. The endogenous gap grows as small domains start being repeated and is clearly separated from the null by $1/10$ of the budget. In all three settings, the diagnosis does not change at later checkpoints. Thus, $\Delta$ need not be measured near the end of training: in our settings, about $1/10$ of the budget suffices, although the required point may depend on the task.

\begin{table}[h]
\centering
\caption{$\Delta$ of the Uniform baseline in the mixture settings at several checkpoints (mean$\pm$std over five seeds). The pretrained Pythia-160M is used as the reference. The total budget is 40K steps.}
\label{tab:delta_mixture_ref}
\begin{tabular}{lcccc}
\toprule
Budget & $1/20$ & $1/10$ & $3/20$ & $1/2$ \\
\midrule
No gap (null) & $0.026\pm0.003$ & $0.003\pm0.001$ & $0.013\pm0.001$ & $0.076\pm0.001$ \\
Endogenous & $0.041\pm0.003$ & $0.102\pm0.002$ & $0.158\pm0.002$ & $0.649\pm0.013$ \\
\quad $\zeta$ & $5$ & $155$ & $147$ & $537$ \\
Exogenous & $0.448\pm0.012$ & $0.465\pm0.002$ & $0.458\pm0.004$ & $0.468\pm0.002$ \\
\quad $\zeta$ & $137$ & $726$ & $452$ & $367$ \\
\bottomrule
\end{tabular}
\end{table}

\begin{figure*}[t]
\centering
\begin{minipage}{0.48\linewidth}
    \centering
    \subfigure[]{
    \includegraphics[width=\linewidth]{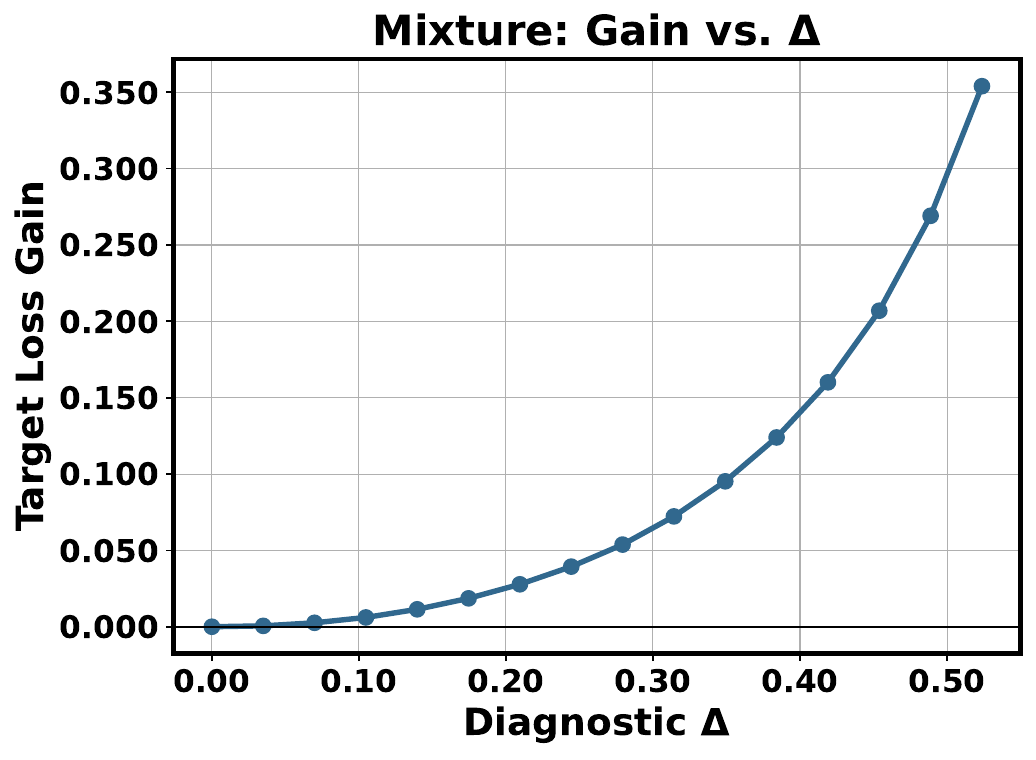}}
\end{minipage}
\hfill
\begin{minipage}{0.48\linewidth}
    \centering
    \subfigure[]{
    \includegraphics[width=\linewidth]{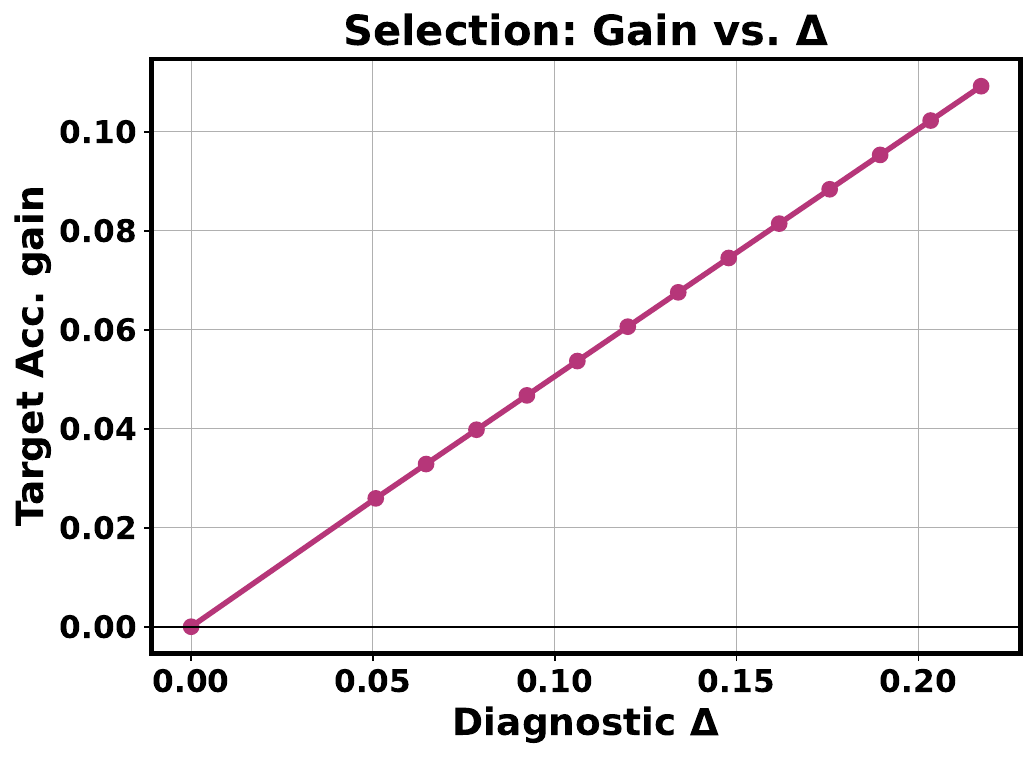}}
\end{minipage}
\caption{\textbf{Curation gain versus the $\Delta$ diagnostic in the solvable models.}
Each point varies the parameter that generates the train--target gap and reports the improvement of the optimized policy over the uncurated baseline.
(a) Mixture: target-composition sweep. Gain is the reduction in target logistic loss. (b) Selection: teacher-angle sweep. Gain is the reduction in target classification error.}
\label{fig:replica-gain-vs-delta}
\end{figure*}

\subsection{Real-Data Experimental Details}
\label{app:real_data_details}
This section collects the implementation details. Across settings, we evaluate all baselines within a shared protocol and report five independent runs. The diagnostic $\Delta$ is measured at 50\% training budget.

\subsubsection{Data-mixture experiments}
\paragraph{Setup and evaluation.}
Following \citet{chen2025aioli}, we use a 6B-token subset of
SlimPajama~\citep{soboleva2023slimpajama}, comprising ArXiv, Books,
CommonCrawl, C4, GitHub, StackExchange, and Wikipedia. In the no-gap setting,
both training and target weight all seven domains uniformly. In the
exogenous-gap setting, CommonCrawl is held out as the target and the remaining six
domains form the candidate training pool. The endogenous-gap setting retains
the uniformly weighted target but restricts the available corpus to 300M
tokens while preserving its original domain proportions, causing small domains
to be repeatedly sampled. We report held-out perplexity for each target domain.
For a multi-domain target, "Average" is the exponential of the mean domain
loss, following \citet{fan2024doge}. for the exogenous-gap setting, the metric
is CommonCrawl perplexity.

\paragraph{Shared training protocol.}
The main experiments train 160M-parameter GPT-style
models~\citep{biderman2023pythia} for 40K steps. As summarized in \Cref{tab:mixture-shared-config}. All three gap settings use the same batch size, context length, schedule, and number of steps, so the comparison changes the available data or target rather than the base training protocol. Mixture optimization and final-model training use the same model-training hyperparameters unless stated otherwise.

\begin{table}[h]
\caption{Shared configuration for the data-mixture experiments.}
\label{tab:mixture-shared-config}
\centering
\small
\setlength{\tabcolsep}{6pt}
\begin{tabular}{@{}lcc@{}}
\toprule
Configuration & 160M & 1.1B \\ \midrule
Training steps & 40K & 40K \\
Batch size & 8 & 8 \\
Context length & 2048 & 2048 \\
Peak learning rate & $5\times10^{-4}$ & $2\times10^{-4}$ \\
Warmup steps & -- & 500 \\
Weight decay & $0.01$ & $0.01$ \\
Gradient clipping & 1.0 & 1.0 \\ \bottomrule
\end{tabular}
\end{table}

\paragraph{Baselines and implementation.}
All methods use the shared protocol above. Uniform assigns equal exposure to
the candidate domains and serves as the baseline policy. We compare it with
RegMix~\citep{liu2025regmix}, which fits global data mixing laws and predict the performance on other mixtures;
DoGE~\citep{fan2024doge}, which updates domain weights using their estimated influence on the target objective; 
Skill-It~\citep{chen2023skillit} utilize the inter-domain dependencies to establish a local data mixing law that predicts the per-domain performance;
Aioli~\citep{chen2025aioli} improves Skill-It by dynamically updating the relation matrix along model training; and TANDEM~\citep{wang2025tandem}, which optimizes the mixture through a bi-level objective with twinned networks. We fit the RegMix mixing law using 128 independently trained 1M-parameter proxy models. For RegMix, DoGE, Skill-It, and Aioli, we adapt the public implementations to our common data interface and training pipeline.\footnote{Public implementations: RegMix,\url{https://github.com/sail-sg/regmix}; DoGE, \url{https://github.com/Olivia-fsm/DoGE}; and Skill-It and Aioli, \url{https://github.com/HazyResearch/aioli}.} Method-specific choices otherwise follow the corresponding papers and released implementations. The degree of smoothing recommended in the original Skill-It and Aioli settings can yield nearly uniform allocations in our experiments, we adjust only their smoothing controls for each gap setting, using less smoothing in the
endogenous- and exogenous-gap settings. The learned allocation is used to train the final model with the shared configuration above. Implementation and configuration will be released with the complete codebase.

\subsubsection{Data-selection experiments}
\paragraph{Setup and target construction.}
We use TinyLlama-1.1B~\citep{zhang2024tinyllama} as the base model and draw a 5B-token candidate pool from OpenWebMath~\citep{paster2024openwebmath}. The in-distribution target is a held-out 0.3B-token OpenWebMath subset. The shifted target blends MetaMath~\citep{yu2024metamath} and Mammoth~\citep{yu22024mammoth}. The main comparison retains 60\% of the candidate pool. \Cref{fig:selection_results} additionally sweeps the retained fraction from 0.6 to 1.0.

\paragraph{Shared training protocol.}
All selection methods use the configuration in
\Cref{tab:selection-shared-config}. They
share the same base model, candidate pool, retained-token budget, and training
protocol, so that only the selection rule and target distribution vary.

\begin{table}[h]
\caption{Shared configuration for the data-selection experiments.}
\label{tab:selection-shared-config}
\centering
\small
\setlength{\tabcolsep}{6pt}
\begin{tabular}{@{}lc@{}}
\toprule
Configuration & TinyLlama-1.1B \\ \midrule
Candidate-training steps & 5000 \\
Batch size & 512 \\
Context length & 2048 \\
Learning rate & $8\times10^{-5}$ \\
Learning-rate schedule & Constant \\
Weight decay & 0.01 \\
Gradient clipping & 1.0 \\
Main retained fraction & 0.6 \\ \bottomrule
\end{tabular}
\end{table}

\paragraph{Baselines and implementation.}
The original pretrained checkpoint is denoted Base, while Base-CT continually pre-trains it on the full candidate pool and serves as the full-retention baseline. Random controls for the reduced token budget by drawing an equally sized subset. We compare three target-agnostic selectors---EL2N~\citep{mansheej21datadiet}, which prioritizes examples with large prediction
error; InfoBatch~\citep{qin2024infobatch}, which approximately preserves the expected full-data gradient; and Token-Cleaning~\citep{pang2025tokencleaning}, which applies a noisy-label cleaning criterion---with two target-aware selectors: RHO-1~\citep{lin2024rho1}, which ranks tokens by excess loss relative to the target distribution, and BLADE~\citep{wang2026blade}, which performs bi-level optimization for large-scale data selection. For EL2N, InfoBatch, and Token-Cleaning, we adapt their public implementations only as needed for our common data interface and training pipeline.\footnote{Public implementations: EL2N,
\url{https://github.com/mansheej/data_diet}; InfoBatch,
\url{https://github.com/NUS-HPC-AI-Lab/InfoBatch}; and Token-Cleaning,
\url{https://github.com/UCSC-REAL/TokenCleaning}.} For RHO-1 and BLADE, for which no public implementations are available, we implement the methods following their published descriptions. Method-specific settings otherwise follow the corresponding papers or released implementations, while every selected model is trained with the shared configuration in
\Cref{tab:selection-shared-config}.

\paragraph{Evaluation.}
We report held-out target loss, which is evaluated on the distribution used to define $\Delta$ and therefore reflects the train--target gap more directly than downstream accuracy. We additionally report accuracy on eight complementary mathematics benchmarks: GSM8K~\citep{cobbe2021gsm8k}, MATH~\citep{hendrycks2021math}, SVAMP~\citep{patel2021svamp}, ASDiv~\citep{miao2020asdiv}, MAWPS~\citep{kedziorski2016mawps}, TabMWP~\citep{lu2023tab} for tabular reasoning, MathQA~\citep{amini2019mathqa} for operation-based multiple-choice reasoning and MMLU-STEM~\citep{hendrycks2020mmlu}. This evaluation encompasses diverse question types like multiple-choice and open-ended questions.

\subsection{Comparisons with Standard Deviation}
\label{app:real_world_stds}
Each real-world LLM training experiment is run 5 times independently. The results, including the standard deviation, are reported in Table~\Cref{tab:mixture_without_gap_std}, Table~\Cref{tab:mixture_with_gap_std} and Table~\Cref{tab:sota_math_std}.

\begin{sidewaystable}
\caption{Comparison of data mixture optimization methods in the no-gap regime ($\Delta=0.076$). The averaged target perplexity and standard deviation over 5 independent runs are reported.}
\label{tab:mixture_without_gap_std}
\centering
\small
\begin{tabular}{@{}lcccccccc@{}}
\toprule
Methods  & \multicolumn{1}{l}{Arxiv} & \multicolumn{1}{l}{Book} & \multicolumn{1}{l}{C4} & \multicolumn{1}{l}{CommonCrawl} & \multicolumn{1}{l}{Github} & \multicolumn{1}{l}{Stackexchange} & \multicolumn{1}{l}{Wikipedia} & \multicolumn{1}{l}{Average} \\ \midrule                                                                      
Uniform  & 10.02$\pm_{\text{0.11}}$                     & 39.61$\pm_{\text{0.25}}$                    & 49.67$\pm_{\text{0.41}}$                  & 42.70$\pm_{\text{0.38}}$                           & 5.16$\pm_{\text{0.06}}$                       & 12.25$\pm_{\text{0.09}}$                             & 20.00$\pm_{\text{0.22}}$                         & 19.71$\pm_{\text{0.17}}$                       \\
RegMix   &  10.53$\pm_{\text{0.34}}$                    &     41.90$\pm_{\text{1.32}}$                      &    48.67$\pm_{\text{1.85}}$             &     41.65$\pm_{\text{1.59}}$                          &        5.75$\pm_{\text{0.24}}$                     &   11.73$\pm_{\text{0.73}}$                           &      19.86$\pm_{\text{1.19}}$                          &      19.80$\pm_{\text{0.12}}$                        \\
DoGE     &     10.96$\pm_{\text{0.22}}$        &   35.49$\pm_{\text{0.19}}$       &            42.42$\pm_{\text{0.26}}$            &       36.81$\pm_{\text{0.28}}$                          &     6.70$\pm_{\text{0.10}}$                       &       13.97$\pm_{\text{0.17}}$                            &        23.92$\pm_{\text{0.30}}$                       &       20.17$\pm_{\text{0.16}}$                      \\
Skill-It &     10.64$\pm_{\text{0.17}}$                      &     38.32$\pm_{\text{0.21}}$                     &         47.99$\pm_{\text{0.37}}$               &     40.62$\pm_{\text{0.45}}$                            &        5.73$\pm_{\text{0.07}}$                    &          12.07$\pm_{\text{0.13}}$                         &           20.78$\pm_{\text{0.20}}$                    &     19.68$\pm_{\text{0.16}}$                         \\
Aioli    &     10.09$\pm_{\text{0.13}}$                      &        39.86$\pm_{\text{0.35}}$                  &          49.81$\pm_{\text{0.54}}$              &     42.83$\pm_{\text{0.52}}$                            &   5.66$\pm_{\text{0.04}}$                         &    12.36$\pm_{\text{0.06}}$                                &        19.67$\pm_{\text{0.21}}$             &        19.77$\pm_{\text{0.15}}$                     \\
TANDEM   &        10.18$\pm_{\text{0.12}}$                    &  37.89$\pm_{\text{0.32}}$                         &          46.77$\pm_{\text{0.40}}$               &       40.41$\pm_{\text{0.37}}$                           &     5.89$\pm_{\text{0.05}}$                        &    12.61$\pm_{\text{0.12}}$                                &     20.96$\pm_{\text{0.21}}$                           &     19.66$\pm_{\text{0.15}}$                       \\ \bottomrule
\end{tabular}
\vspace{5\baselineskip}

\caption{Comparison of data mixture optimization methods under endogenous ($\Delta =0.649$)and exogenous ($\Delta=0.468$) train--target gaps. The averaged target perplexity and standard deviation over 5 independent runs are reported.}
\label{tab:mixture_with_gap_std}
\centering
\small
\begin{tabular}{@{}lcccccccc|l@{}}
\toprule
\multicolumn{9}{c|}{Endogenous Gap}                 & \multicolumn{1}{c}{Exogenous Gap} \\ \midrule
Methods  & \multicolumn{1}{l}{Arxiv} & \multicolumn{1}{l}{Book} & \multicolumn{1}{l}{C4} & \multicolumn{1}{l}{CommonCrawl} & \multicolumn{1}{l}{Github} & \multicolumn{1}{l}{Stackexchange} & \multicolumn{1}{l}{Wikipedia} & \multicolumn{1}{l|}{Average} & CommonCrawl                \\ \midrule
Uniform  & 13.14$\pm_{\text{0.04}}$ & 57.44$\pm_{\text{0.73}}$ & 65.57$\pm_{\text{0.56}}$ & 57.93$\pm_{\text{0.54}}$ & 8.21$\pm_{\text{0.04}}$  & 19.59$\pm_{\text{0.03}}$ & 33.90$\pm_{\text{0.20}}$ & 28.60$\pm_{\text{0.11}}$                   &          49.19$\pm_{\text{0.28}}$                \\
RegMix  & 13.33$\pm_{\text{0.07}}$ & 54.12$\pm_{\text{0.23}}$ & 63.74$\pm_{\text{0.67}}$ & 56.63$\pm_{\text{0.70}}$ & 8.23$\pm_{\text{0.08}}$ & 20.03$\pm_{\text{0.11}}$ & 33.34$\pm_{\text{0.19}}$ & 28.24$\pm_{\text{0.19}}$ &            40.63$\pm_{\text{0.90}}$                \\
DoGE   & 12.95$\pm_{\text{0.05}}$ & 54.82$\pm_{\text{1.28}}$ & 62.84$\pm_{\text{0.63}}$ & 55.52$\pm_{\text{0.61}}$ & 8.13$\pm_{\text{0.05}}$  & 18.61$\pm_{\text{0.08}}$ & 32.72$\pm_{\text{0.17}}$ & 27.62$\pm_{\text{0.22}}$ &                40.67$\pm_{\text{0.44}}$              \\
Skill-It & 14.54$\pm_{\text{0.13}}$ & 44.24$\pm_{\text{0.26}}$ & 48.99$\pm_{\text{0.31}}$ & 42.10$\pm_{\text{0.32}}$ & 8.33$\pm_{\text{0.10}}$  & 21.10$\pm_{\text{0.19}}$ & 49.46$\pm_{\text{0.54}}$ & 27.38$\pm_{\text{0.22}}$  &            41.08$\pm_{\text{0.17}}$                \\
Aoili & 12.80$\pm_{\text{0.14}}$ & 51.01$\pm_{\text{0.37}}$ & 61.13$\pm_{\text{0.31}}$ & 53.51$\pm_{\text{0.34}}$ & 8.14$\pm_{\text{0.05}}$  & 18.97$\pm_{\text{0.14}}$ & 32.63$\pm_{\text{0.13}}$ & 27.11$\pm_{\text{0.16}}$ &             40.17$\pm_{\text{0.64}}$                \\
TANDEM        & 12.60$\pm_{\text{0.12}}$ & 49.76$\pm_{\text{1.56}}$ & 54.19$\pm_{\text{0.76}}$ & 47.24$\pm_{\text{0.90}}$ & 8.35$\pm_{\text{0.05}}$  & 17.98$\pm_{\text{0.11}}$ & 33.51$\pm_{\text{0.41}}$ & 26.02$\pm_{\text{0.22}}$      &        40.24$\pm_{\text{0.38}}$                    \\ \bottomrule
\end{tabular}
\end{sidewaystable}

\begin{sidewaystable}
\caption{Comparison of data selection methods with an in-distribution OpenWebMath target (no gap, $\Delta=0.016$, upper) and the shifted MetaMath--Mammoth reasoning target (with gap, $\Delta=0.158$, lower). Results on 8 downstream tasks and the averaged accuracy are reported. Each experiment is carried out 5 times. }
\label{tab:sota_math_std}
\setlength\tabcolsep{2.8pt}
\small
\centering
\begin{tabular}{@{}lcccccccccc@{}}
\toprule
                & \multicolumn{1}{l}{Train Toks.} & \multicolumn{1}{l}{GSM8K} & \multicolumn{1}{l}{MATH} & \multicolumn{1}{l}{SVAMP} & \multicolumn{1}{l}{ASDiv} & \multicolumn{1}{l}{MAWPS} & \multicolumn{1}{l}{TAB} & \multicolumn{1}{l}{MQA} & \multicolumn{1}{l}{MMLU-STEM} & \multicolumn{1}{l}{Avg.} \\ \midrule
\multicolumn{11}{c}{Selection without Gap}                                                                                                                                                                                                                                                                  \\ \midrule
Base            & -                               & 2.9                      & 3.2                      & 11.0                      & 18.1                      & 20.4                      & 12.5                    & 14.6                    & 16.1                          &      12.4              \\
Base-CT         & 5B                             & 4.3$\pm_{\text{0.5}}$    & 4.3$\pm_{\text{0.2}}$    & 17.1$\pm_{\text{0.7}}$   & 29.4$\pm_{\text{0.6}}$   & 39.4$\pm_{\text{1.2}}$   & 14.5$\pm_{\text{2.0}}$  & 8.3$\pm_{\text{0.6}}$   & 20.4$\pm_{\text{0.8}}$        & 17.21$\pm_{\text{0.4}}$                    \\
Random          & 3B                              & 4.0$\pm_{\text{0.3}}$    & 3.9$\pm_{\text{1.0}}$    & 15.8$\pm_{\text{1.0}}$   & 29.1$\pm_{\text{2.3}}$   & 37.2$\pm_{\text{1.6}}$   & 13.8$\pm_{\text{0.7}}$  & 9.4$\pm_{\text{1.1}}$   & 17.8$\pm_{\text{2.2}}$        & 16.38$\pm_{\text{0.1}}$                    \\
EL2N            & 3B  & 3.0$\pm_{\text{0.3}}$    & 3.2$\pm_{\text{0.5}}$    & 15.5$\pm_{\text{0.9}}$   & 25.0$\pm_{\text{1.2}}$   & 31.2$\pm_{\text{1.1}}$   & 14.4$\pm_{\text{1.7}}$  & 8.9$\pm_{\text{1.6}}$   & 18.3$\pm_{\text{1.2}}$        & 14.94$\pm_{\text{0.4}}$                    \\
InfoBatch       & 3B  & 3.8$\pm_{\text{0.4}}$    & 3.3$\pm_{\text{0.6}}$    & 15.8$\pm_{\text{1.5}}$   & 25.8$\pm_{\text{1.5}}$   & 32.9$\pm_{\text{2.3}}$   & 12.5$\pm_{\text{1.3}}$  & 7.1$\pm_{\text{0.9}}$   & 16.5$\pm_{\text{2.3}}$        & 14.71$\pm_{\text{0.7}}$                    \\
Token Cleaning  & 3B                        & 4.2$\pm_{\text{0.6}}$    & 3.0$\pm_{\text{0.9}}$    & 15.4$\pm_{\text{0.7}}$   & 24.4$\pm_{\text{0.6}}$   & 31.7$\pm_{\text{1.0}}$   & 11.7$\pm_{\text{0.7}}$  & 10.2$\pm_{\text{1.2}}$   & 14.1$\pm_{\text{0.9}}$        & 14.3$\pm_{\text{0.5}}$                    \\
RHO-1 & 3B  & 3.2$\pm_{\text{0.2}}$    & 3.1$\pm_{\text{1.3}}$    & 15.3$\pm_{\text{0.5}}$   & 24.6$\pm_{\text{0.7}}$   & 31.4$\pm_{\text{0.5}}$   & 14.5$\pm_{\text{1.9}}$  & 7.1$\pm_{\text{2.3}}$   & 16.7$\pm_{\text{2.1}}$        & 14.49$\pm_{\text{0.3}}$                    \\
BLADE           & 3B     & 4.5$\pm_{\text{0.3}}$    & 3.7$\pm_{\text{0.4}}$    & 18.1$\pm_{\text{0.6}}$   & 29.9$\pm_{\text{1.7}}$   & 37.9$\pm_{\text{1.7}}$   & 16.5$\pm_{\text{2.1}}$  & 10.4$\pm_{\text{1.5}}$   & 19.3$\pm_{\text{2.1}}$        & 17.54$\pm_{\text{0.8}}$                    \\
\midrule \midrule
\multicolumn{11}{c}{Selection with Gap}                                                                                                                                                                                                                                                                     \\ \midrule
Base            & -                               & 2.9                      & 3.2                      & 11.0                      & 18.1                      & 20.4                      & 12.5                    & 14.6                    & 16.1                          &      12.4              \\
Base-CT         & 5B                             & 4.3$\pm_{\text{0.5}}$    & 4.3$\pm_{\text{0.2}}$    & 17.1$\pm_{\text{0.7}}$   & 29.4$\pm_{\text{0.6}}$   & 39.4$\pm_{\text{1.2}}$   & 14.5$\pm_{\text{2.0}}$  & 8.3$\pm_{\text{0.6}}$   & 20.4$\pm_{\text{0.8}}$        & 17.21$\pm_{\text{0.4}}$                    \\
Random          & 3B                              & 4.0$\pm_{\text{0.3}}$    & 3.9$\pm_{\text{1.0}}$    & 15.8$\pm_{\text{1.0}}$   & 29.1$\pm_{\text{2.3}}$   & 37.2$\pm_{\text{1.6}}$   & 13.8$\pm_{\text{0.7}}$  & 9.4$\pm_{\text{1.1}}$   & 17.8$\pm_{\text{2.2}}$        & 16.38$\pm_{\text{0.1}}$                    \\
EL2N            & 3B  & 3.0$\pm_{\text{0.3}}$    & 3.2$\pm_{\text{0.5}}$    & 15.5$\pm_{\text{0.9}}$   & 25.0$\pm_{\text{1.2}}$   & 31.2$\pm_{\text{1.1}}$   & 14.4$\pm_{\text{1.7}}$  & 8.9$\pm_{\text{1.6}}$   & 18.3$\pm_{\text{1.2}}$        & 14.94$\pm_{\text{0.4}}$                    \\
InfoBatch       & 3B  & 3.8$\pm_{\text{0.4}}$    & 3.3$\pm_{\text{0.6}}$    & 15.8$\pm_{\text{1.5}}$   & 25.8$\pm_{\text{1.5}}$   & 32.9$\pm_{\text{2.3}}$   & 12.5$\pm_{\text{1.3}}$  & 7.1$\pm_{\text{0.9}}$   & 16.5$\pm_{\text{2.3}}$        & 14.71$\pm_{\text{0.7}}$                    \\
Token Cleaning  & 3B                        & 4.2$\pm_{\text{0.6}}$    & 3.0$\pm_{\text{0.9}}$    & 15.4$\pm_{\text{0.7}}$   & 24.4$\pm_{\text{0.6}}$   & 31.7$\pm_{\text{1.0}}$   & 11.7$\pm_{\text{0.7}}$  & 10.2$\pm_{\text{1.2}}$   & 14.1$\pm_{\text{0.9}}$        & 14.3$\pm_{\text{0.5}}$                    \\
RHO-1 & 3B    & 15.1$\pm_{\text{1.1}}$    & 5.9$\pm_{\text{0.6}}$    & 32.6$\pm_{\text{1.3}}$   & 46.9$\pm_{\text{0.7}}$   & 63.8$\pm_{\text{0.5}}$   & 18.7$\pm_{\text{2.1}}$  & 15.1$\pm_{\text{1.0}}$   & 18.4$\pm_{\text{2.5}}$        & 27.06$\pm_{\text{0.4}}$                     \\
BLADE           & 3B     & 21.2$\pm_{\text{0.8}}$    & 8.1$\pm_{\text{1.3}}$    & 37.0$\pm_{\text{1.3}}$   & 51.4$\pm_{\text{0.4}}$   & 70.2$\pm_{\text{1.1}}$   & 21.2$\pm_{\text{0.2}}$  & 20.6$\pm_{\text{1.7}}$   & 20.8$\pm_{\text{0.9}}$        & 31.31$\pm_{\text{0.6}}$                    \\
\bottomrule
\end{tabular}
\end{sidewaystable}

\subsection{Limitations and Future Work}
\label{app:limitations}
While our analytic characterization and empirical results provide a unified, gap-centric account of data curation, several challenges remain open for future research. The limitations of this paper are as follows: (1) Due to limited computational resources, our real-world experiments evaluate models up to 1.1 billion parameters. (2) Our analytic characterization relies on static high-dimensional limits that abstract away LLM training dynamics and data curriculum order. While our theory cleanly distinguishes finite-budget acceleration from genuine advancement, extending these solvable models to capture precise convergence speeds or detailed data consumption trajectories presents an interesting direction for future investigation.

\subsection{Notation}
\label{app:notation}

\begin{table}[h]
\caption{Summary of the notations used throughout this paper. Variables only used in theoretical
analysis are grayed for better readability.}
\label{tab:notation}
\centering
\small
\renewcommand{\arraystretch}{1.25}
\begin{tabular}{p{0.28\linewidth} p{0.66\linewidth}}
\toprule
\textbf{Symbol} & \textbf{Meaning} \\
\midrule
\multicolumn{2}{l}{\emph{Data and dimensions}}\\
$N$ & input dimension. \\
$P$ & number of training examples. \\
\rowcolor{gray!40} $\alpha=P/N$ & total sample budget (samples per dimension). In high-dim limit $N,P\to\infty$, $\alpha$ fixed. \\
\rowcolor{gray!40} $\alpha_k=w_k\alpha$ & sample budget contributed by domain $k$. \\
$\bx\in\mathbb R^N$ & input data. \\
\rowcolor{gray!40} $y=\mathrm{sign}(\bB\cdot\bx)$ & label generated by a teacher. \\
$K$ & number of domains. \\
$\mathcal N(z;0,1)=e^{-z^2/2}/\sqrt{2\pi}$ & standard Gaussian density. \\
\midrule
\multicolumn{2}{l}{\emph{Teachers, student, model}}.\\
\rowcolor{gray!40} $\bB_k$ & teacher vector of domain $k$ ($\|\bB_k\|=\sqrt N$). \\
\rowcolor{gray!40} $\bB_{\text{tar}}$ & target teacher.  \\
\rowcolor{gray!40} $\phi$ & teacher angle (exogenous gap) between $\bB_k$ and the target. \\
\rowcolor{gray!40} $\bJ$ & student weights $=$ model parameters (replica symbol). \\
$\btheta$ & model parameters ($\btheta\equiv\bJ$). \\
\rowcolor{gray!40} $\rho$ & student--target overlap $\bJ\cdot\bB_{\mathrm{tar}}/N$ in Appendix~\ref{app:imperfect}. \\
\midrule
\multicolumn{2}{l}{\emph{Curation controls}}\\
$\bw$ & curation weights: mixture ratios $w_k$ ($\sum_k w_k=1$) or per-example keep weights. \\
$\widehat{\bw}$ & target mixture ratio. \\
$f$ & retained fraction of the candidate set. \\
\rowcolor{gray!40} $\lambda_{\mathrm{reg}}$ & regularization coefficient in the mixture model. \\
\rowcolor{gray!40} $\lambda_{\mathrm{sel}}$ & disagreement penalty in the selection policy. \\
\rowcolor{gray!40} $\nu^2$ & difficulty: nuisance variance orthogonal to a teacher (Appendix~\ref{app:mixture}) \\
\midrule
\multicolumn{2}{l}{\emph{Order parameters (state evolution / replica)}}\\
\rowcolor{gray!40} $m_k=\bJ\cdot\bB_k/N$ & student--teacher overlap for domain $k$. \\
\rowcolor{gray!40} $q=\|\bJ\|^2/N$ & student self-overlap. \\
\rowcolor{gray!40} $V,\,V_k$ & response term. $V_k$ per domain in the multi-environment case. \\
$V_k^\star=1-m_k^2/Q_k$ & teacher-conditional variance in the loss integrals ($Q_1=q$, $Q_2=q+\nu^2(q-m_2^2)$). \\
\rowcolor{gray!40} $\kappa$ & margin constraint parameter (hard-margin solver). \\
\rowcolor{gray!40} $R$ & student--training-teacher overlap; in the matched-target case it is also the target overlap. \\
\rowcolor{gray!40} $\hat m,\hat V,\hat q$ & conjugate (hat) overlaps in the SE fixed point. \\
\midrule
\multicolumn{2}{l}{\emph{Performance and diagnostic}}\\
$\mathcal E_\alpha(\bw)$ & target performance at sample budget $\alpha$ under allocation $\bw$; for selection we write $\mathcal E_\alpha(f;\phi)$. \\
$\varepsilon$ & fixed target-performance level used to define acceleration. \\
$\Delta$ & empirical train--target diagnostic on the baseline model (\Cref{sec:delta_test}) \\
\bottomrule
\end{tabular}
\end{table}

\end{document}